\documentclass[12pt]{article}
\usepackage{arxiv,times}
\usepackage[utf8]{inputenc}
\usepackage[T1]{fontenc}
\usepackage{natbib}
\setcitestyle{authoryear,round,citesep={;},aysep={,},yysep={;}}

\usepackage{amsmath,amsfonts,bm}

\def\eqref#1{equation~\ref{#1}}

\def\1{\bm{1}}

\def\eps{{\epsilon}}

\def\rve{{\mathbf{e}}}

\def\rvu{{\mathbf{i}}}

\def\rvu{{\mathbf{u}}}
\def\rvv{{\mathbf{v}}}
\def\rvw{{\mathbf{w}}}

\DeclareMathAlphabet{\mathsfit}{\encodingdefault}{\sfdefault}{m}{sl}
\SetMathAlphabet{\mathsfit}{bold}{\encodingdefault}{\sfdefault}{bx}{n}

\usepackage{hyperref}
\usepackage{url}
\usepackage{amsmath,amssymb,amsthm,mathtools,bm}
\usepackage{graphicx,booktabs,microtype}
\usepackage{xcolor}
\usepackage[nameinlink,capitalize,noabbrev]{cleveref}

\newtheorem{theorem}{Theorem}
\newtheorem{lemma}{Lemma}
\newtheorem{proposition}{Proposition}

\title{
  \fontsize{17pt}{20pt}\selectfont
  Probability‑Signature Dynamics: \\
  Unpacking Modular Addition Learning Within Two‑Layer Networks
}

\author{
  \parbox{0.95\textwidth}{
    \centering
    \normalfont
    \textbf{Yunji Wang}$^{1,\dagger}$,\quad
    \textbf{Junjie Yao}$^{1,\dagger}$,\quad
    \textbf{Linyu Liu}$^{3}$,\quad
    \textbf{Pinyan Lu}$^{4}$,\quad
    \textbf{Zhi-Qin John Xu}$^{1,2,*}$
    \\[0.8em]
    $^{1}$School of Mathematical Sciences,
    Shanghai Jiao Tong University\\[0.3em]
    $^{2}$Institute of Natural Sciences, MOE-LSC,
    Shanghai Jiao Tong University\\[0.3em]
    $^{3}$Huawei Technologies Ltd\\[0.3em]
    $^{4}$Shanghai University of Finance and Economics\\[0.3em]
    \small{
    $^{\dagger}$These authors contributed equally to this work.\\[0.2em]
    $^{*}$Corresponding author: \texttt{xuzhiqin@sjtu.edu.cn}}
  }
}

\date{}
\renewcommand{\undertitle}{}
\renewcommand{\shorttitle}{Probability-Signature Dynamics}

\begin{document}
\maketitle

\begin{abstract}
Neural networks trained on modular addition tasks often develop Fourier-structured representations that support exact generalization. While prior work has identified these Fourier circuits, the mechanism by which gradient-based training selects them from the data distribution remains unclear. We address this question using probability signatures, which express leading gradient interactions through conditional statistics of the training distribution. For modular addition, these signatures are cyclic shift operators and are diagonalized by the discrete Fourier transform, yielding approximately decoupled Fourier-mode dynamics. This explains the emergence of Fourier sparsity, frequency matching, and phase alignment. 
The same framework resolves a puzzle under label noise: corrupted examples can show faster early loss decrease than clean examples, despite lacking a coherent generalization rule. We show that noise increases conditional label collisions, strengthening early shared-coordinate reinforcement. Finally, this method can be applied to other operators. Taking XOR as an example, we observed the predicted frequency in experiments.
\end{abstract}


\section{Introduction}

The modular addition task is a canonical setting in which to study grokking. During the initial phase of training, an over-parameterized model can fit individual training examples yet fails to uncover an interpretable rule. After substantially prolonged training, however, its test performance may improve abruptly even though the training loss has already become very small. This delayed transition is commonly referred to as \emph{grokking}. In modular addition, the onset of grokking is accompanied by the emergence of a distinctive structure in the weights. Prior studies have shown that trained networks develop Fourier circuits \citep{nanda2023progressmeasuresgrokkingmechanistic,gromov2023grokkingmodulararithmetic}. Specifically, the input and output weights of many hidden units become concentrated at matching Fourier frequencies, and appropriate combinations of these Fourier components can compute $(a+b)\bmod P$. Mechanistic interventions, including Fourier projection and ablation experiments, provide strong evidence that these components are functionally important for generalization. Their formation also reflects the transition from a memorizing solution to a generalizing one.

These results answer an important \emph{reverse} question: given a trained network, which circuit does it implement, and which components of that circuit support generalization? They have not fully answered the complementary \emph{forward} question: why should gradient-based training produce Fourier-structured parameters from the modular-addition data distribution in the first place? The fact that Fourier features can represent modular addition does not imply that the training dynamics must discover them. Meanwhile, current work can already account for observed phenomena in neural networks, such as frequency competition and phase alignment~\citep{he2026mechanismdynamicsmodularaddition}. However, its theoretical analysis is based on a Fourier decomposition of the weights, which need to presuppose that the Fourier basis is related to modular addition.

Our focus is therefore not to show again that Fourier-structured neurons can implement modular addition, nor to re-establish that these neurons support generalization. Instead, we study the missing forward direction
\[
\text{training distribution}
\quad\Longrightarrow\quad
\text{leading gradient terms}
\quad\Longrightarrow\quad
\text{Fourier-mode selection}.
\]
More specifically, we ask the following questions. Which statistics of the training data enter the leading-order parameter dynamics? Why are Fourier modes approximately decoupled in the modular-addition problem?
What determines whether a neuron is suppressed by weight decay or develops a dominant Fourier component? And, how to derive the weight structure in advance from the training distribution?

We study these questions in a two-layer embedding MLP trained on modular addition dataset with some randomly corrupted labels. Our main analytical tool is the \emph{probability signature} framework introduced by \citet{yao2025probabilitysignature}, which shows the correlation between data statistics and embedding without studying either modular addition learning or the key structure of data statistics. 

Our contributions are as follows.
\begin{itemize}
\item  We utilize probability signatures to directly explain the emergence of dominant Fourier modes during the learning dynamics.
\item We have explained the puzzle in the early memorization regime: uniformly random label noise consistently exhibits a faster initial loss decrease than clean data. Some statistics from training distribution can measure this difference.
\item The analysis framework can be extended to more tasks, such as XOR, enabling us to predict the structures that network parameters may form before training.
\end{itemize}

\section{Related Work}

\textbf{Grokking.}
Grokking, the delayed onset of generalization long after interpolation, was first reported on small algorithmic datasets, where weight decay proved essential \citep{power2022grokkinggeneralization}. Subsequent work explains the delay through properties of the optimization trajectory: effective theories of representation learning and norm growth \citep{liu2022towardsunderstandinggrokking}, loss landscapes shaped by weight decay and initialization scale \citep{liu2022omnigrok}, instabilities of adaptive optimizers \citep{thilak2022slingshotmechanism}, and the transition from lazy to rich training dynamics \citep{kumar2024grokkingtransitionlazyrich}. These characterize \emph{when} and \emph{why} generalization occurs, from the emergent Fourier circuits as a result of training. 
Based on this, we point out that weight decay determines the formation of the Fourier structure as a \emph{threshold}, which establishes the connection between weight decay and Grokking.

\textbf{Modular arithmetic.}
Prior work has established that trained networks can implement modular addition through sparse Fourier circuits, and has introduced diagnostics
such as frequency matching and phase alignment to characterize these circuits \citep{nanda2023progressmeasuresgrokkingmechanistic}.
Analytic and scaling-oriented studies have further related the norms and spectral profiles of Fourier modes to the learned solution \citep{gromov2023grokkingmodulararithmetic,michaud2023quantizationmodelneuralscaling}. These works provide important mechanistic and functional descriptions of Fourier structure in trained networks. 
Our analysis is based on these research results, but we present a forward perspective: starting from \textbf{training distribution} rather than from the Fourier circuits obtained through training. More recent work \citep{he2026mechanismdynamicsmodularaddition} provides a complete explanation for the frequency competition and phase alignment by decomposing neurons onto the Fourier basis. Compared with theirs, our contribution stands out in the fact that we do not require \textbf{the prior knowledge of} the Fourier circuit found in the modular-addition task. By studying the algebraic structure of the \emph{probability signature}, we derive the Fourier basis as the common eigenvectors. We not only explain the above phenomena, but also we hope to provide a unified explanation for the puzzle of early noise and the extension to other training tasks.

\textbf{Learning with corrupted labels.}
Neural networks can both generalize from structured examples and memorize
corrupted labels. This interaction has been studied broadly through early
learning and memorization dynamics
\citep{arpit2017closerlookmemorization,liu2020earlylearningregularization}.
For algorithmic and modular-arithmetic tasks, prior work has documented robust
generalization in the presence of label corruption and the coexistence of
generalizing and memorizing subnetworks
\citep{doshi2024togroknotgrok,liu2026unveilingmemorizationgeneralizationcoexistencecase}.
Our contribution is not a general theory of learning under label noise.
Instead, we isolate one early-time phenomenon: randomly corrupted examples lead to faster initial loss decrease, and we relate it to probability signature, which is computable from the training data.

\section{Preliminaries}
\subsection{Settings}
Let $\Omega=\{0,1,\ldots,P-1\}$ be a finite alphabet, where $P$ is a prescribed prime. We consider the prediction of a binary operation $F:\Omega\times \Omega\to \Omega$ from examples $X=(a,b)$ to the label $F(a,b)$. 

All pairs in $\Omega\times\Omega$ are divided into three disjoint subsets $\mathcal{D}_{\text{clean}},\mathcal{D}_{\text{noise}},\mathcal{D}_{\text{test}}$ of cardinalities $N_{\text{clean}}$, $N_{\text{noise}}$, and $N_{\text{test}}$, respectively. The labels of clean and test are set to $F(a,b)=(a+b) \mod P$, while the labels of noise are replaced by uniform random numbers. 

The training dataset consists of $\mathcal{D}_{\text{clean}}$ and $\mathcal{D}_{\text{noise}}$. We define $N_{\text{train}}=N_{\text{clean}}+N_{\text{noise}}$, where $N_{\text{train}}:N_{\text{test}}=8:2$ and $N_{\text{noise}}=\alpha N_{\text{train}}$. $\alpha$ represents the noise ratio.

\textbf{Model.} Let $\rve_a,\rve_b\in\mathbb{R}^P$ be one-hot encodings. We study the two-layer MLP 
$f(a,b)=W\, \sigma (U\rve_a+V\rve_b)\in\mathbb{R}^P$,
where $U,V\in\mathbb{R}^{D\times P}$, $W\in\mathbb{R}^{P\times D}$, $\rve_a,\rve_b$ means the one-hot encodings of input $a,b$, and $D\gg P$. The output distribution is $\operatorname{softmax}(f(a,b))$.

We use the cross-entropy loss function and employ $L^2$-weight decay with coefficient $\lambda$. Parameters are initialized independently according to
$U_{ij},V_{ij},W_{ij}\sim\mathcal{N}(0,D^{-2\gamma}),$
where $\gamma$ controls the initialization scale. Consider a continuous gradient flow model as an approximation, it is $\frac{\mathrm d\theta}{\mathrm dt}=-\eta (\nabla_\theta L+\lambda\theta)$. For more details about initialization, learning rate or optimizer, please refer to the appendix~\ref{app:train}.

\subsection{Probability signatures}

The Probability Signature proposed in the related work \citep{yao2025probabilitysignature} is a statistical quantity derived from sampling and not dependent on training. It can characterize the main terms of the gradient at different training stages. In order to create the application environment for Probability Signature, we first extended the fully connected network with the same embedding to our model.

\begin{proposition}[Gradient flow]
Assume the neural network $f(x,y)=W\sigma(U\rve_x+V\rve_y)$, using full batch GD with cross-entropy loss and ignoring weight decay, with the activation function $\sigma(h)=0+c_1h+\frac{c_2}2 h^2+o(h^2)$ which is $C^3$ near zero. If we sample training data from $(x,y,z) \sim \pi$, then we denote the probability that the first position is $x$ by $r_x$ and the random variable $\mathbb E_\pi[\cdot|x]$ by the conditional expectation under the fixed first position. 

Under the assumption of \textbf{small initialization} that  $\lVert{U}\rVert_\infty ,\lVert{V}\rVert_\infty ,\lVert{W}\rVert_\infty \le\epsilon$, we derived that
\[
\begin{aligned}
  \frac{\mathrm dU^x}{\mathrm dt}=r_x\eta\Bigg\{&
  c_1 W^\top(I-\frac1P\mathbf1_P\mathbf1_P^\top)\left[\mathbb E_\pi[\rve_z\mid x]
  -\frac{c_1}{P}W\bigl(U^x+\mathbb E_\pi[V^y\mid x]\bigr)\right]\\
  &+c_2\,\mathbb E_\pi\left[
  \bigl(U^x+V^y\bigr)\odot W^\top(\rve_z-\frac1P \mathbf1_P)\,\middle|\,x
  \right]\Bigg\}+\delta .
\end{aligned}
\]
where $U^x\in \mathbb{R}^d$ represent the $x$-th column of $U$ and $\lVert \delta\rVert_\infty=O(\eps^3)$ means the higher order terms.
\label{prop:gradient_flow}
\end{proposition}
\emph{Proof.} See Appendix \ref{app:grad-flow}.

From the gradient flow equation, we can further decompose the conditional expectation. Assume sampling from the training distribution $(X,Y,Z) \sim \pi$, we define the probability signatures:
\[
\varphi^{(a)}(x,z)=\mathbb{P}(Z=z\mid X=x),\qquad
\varphi^{(b)}(x,z)=\mathbb{P}(Z=z\mid Y=x),
\]
the 3-order tensor joint probability signatures:
\[
\Phi^{(a)}(x,y,z)=\mathbb{P}(Z=z,Y=y\mid X=x),\quad
\Phi^{(b)}(x,y,z)=\mathbb{P}(Z=z,X=y\mid Y=x),
\]
and the trace-back probability signatures
\[
\psi^{(a)}(x,y)=\mathbb{P}(Y=y\mid X=x),\qquad
\psi^{(b)}(x,y)=\mathbb{P}(X=y\mid Y=x).
\]
We also denote $\varphi^{(a)}_x:=\sum_{z\in \Omega} \mathbb{P}(Z=z\mid X=x) e_z$, a column vector, and $\Phi^{(a)}_x:=\sum_{y,z\in \Omega} \mathbb{P}(Z=z,Y=y\mid X=x) e_y\otimes e_z$, a matrix. The same notation applies to the others. Then we can rewrite the gradient flow equation in Proposition~\ref{prop:gradient_flow} as
\[
\begin{aligned}
\frac{\mathrm dU^x}{\mathrm dt}
=
r_x\eta \Bigg\{
&
c_1W^\top \bigg[ ( \varphi_x^{(a)}
-\frac{1}{P}\mathbf 1_P) -
\frac{c_1}{P}
(I-\frac{1}{P}\mathbf 1_P\mathbf 1_P^\top) W
(U^x+V\psi_x^{(a)}) \bigg]
\\
&+
c_2
\bigg[
U^x\odot W^\top \varphi_x^{(a)} - \frac1P (U^x+V\psi_x^{(a)}) \odot (W^\top\mathbf 1_P)
+
\operatorname{diag}(V\Phi^{(a)}_xW)
\bigg]
\Bigg\}
+
\delta.
\end{aligned}
\]
The above statistics $\varphi,\Phi,\psi$ are referred to as Probability Signatures.

\section{Theoretical Analysis}
In this chapter, 
we hope to derive this frequency structure from a forward perspective. Reviewing the gradient flow equation \ref{prop:gradient_flow}, we divide it into $\varphi$-part, $\Phi$-part and $\psi$-part based on the association with different probability signatures. Additionally, due to the weight decay, there is also a $-\lambda U$ term. The proof of the following theorems will be found in Appendix \ref{app:dominance},\ref{app:add-dynamics}, and \ref{app:thm-dominant}. 

Theorem \ref{thm:dominance} characterizes the dominant part of the gradient flow equation.

\begin{theorem}[Dominance of the gradient]
\label{thm:dominance}
Under the condition of \ref{prop:gradient_flow} and non-zero weight decay, assume that $\eta=1$ and
\[
\varphi^{(a)}_x=\varphi^{(b)}_x=\psi^{(a)}_x=\psi^{(b)}_x=\frac{1}{P}\mathbf 1_P,\qquad \forall x\in \Omega,
\]
Under small initialization that $\lVert{U}\rVert_\infty ,\lVert{V}\rVert_\infty ,\lVert{W}\rVert_\infty \le \epsilon$ we have
\[
\begin{aligned}
\frac{\mathrm dU^x}{\mathrm dt}
&=-\lambda U^x+c\bigg[ \operatorname{diag}(V\Phi^{(a)}_xW)- \overline{V}\odot \overline{W} \bigg]+\delta_U, \\
\frac{\mathrm dV^x}{\mathrm dt}
&=-\lambda V^x+c\bigg[  \operatorname{diag}(U\Phi^{(b)}_xW)- \overline{U}\odot \overline{W} \bigg]+\delta_V, \\
\frac{\mathrm dW_z}{\mathrm dt}
&=-\lambda W_z+c\bigg[ \operatorname{diag}(U\Psi_z V^\top)- \overline{U}\odot \overline{V} \bigg]+\delta_W.
\end{aligned}
\]
where $c=c_2r_x,\ \lVert\delta\rVert_\infty \lesssim \epsilon^3$. Here we define $\overline{U}:=P^{-1}\sum_xU^x,\ \overline{V}:=P^{-1}\sum_xV^x,\ \overline{W}:=P^{-1}\sum_x(W^\top)^x$, and $\Psi_{z}:=\sum_{x,y\in \Omega} \mathbb{P}_\pi[X=x,Y=y\mid Z=z] e_x \otimes e_y$.
\end{theorem}

\subsection{The Fourier basis derived from the signatures}\label{sec:signature-basis}
Let's assume there is no noise and the training is conducted on the entire clean dataset.
We define $\Phi_x^+:=P^{-1}\sum_{j\in \mathbb Z_P} e_j\otimes e_{j+x}$ as the joint signature in this situation.
Write $S_x=P\Phi_x^{+}$, \textbf{which is a shift-$x$ permutation matrix}. They satisfy
\[
S_xS_y=S_{x+y},\qquad S_x^*=S_{-x},\qquad S_x^*S_x=I.
\]
Here $^*$ denotes the conjugate transpose and $I$ is the $P\times P$ identity matrix.
They form a commuting family of normal operators and therefore possess a common orthonormal eigenvectors over $\mathbb C$. The basis can be obtained directly from the simplest nontrivial shift: $S_1q=\omega q$ implies $F(y+1)=\omega F(y)$ for some unit root $\omega^P=1$, which yields
\begin{equation}
F_r(s)=P^{-1/2}e^{-2\pi \mathrm irs/P},\qquad
\Phi_x^{+}F_r=P^{-1}e^{2\pi \mathrm irx/P}F_r.
\end{equation}
These vectors are the normalized Fourier basis, and \(F:=[F_0,F_1,\cdots,F_{P-1}]\) is the standard Fourier matrix. We use $\hat{\rvu}:=F\rvu$ to represent the Discrete Fourier Transform (DFT) of vector $\rvu$.

Estimate by the main term of Theorem \ref{thm:dominance}, we find that this system is separable for each hidden unit $d\in [D]$. Thus we can consider this equation on one hidden state $\rvu,\rvv,\rvw\in \mathbb{R}^{P}$, with system
$$
\dot{\rvu}_x= c\bigg[\underbrace{\operatorname{diag}(\rvv^\top \Phi^{+}_x \rvw)}_{\nu_x}- \underbrace{ (P^{-1}\sum_{k=0}^{P-1} \rvv_k)(P^{-1}\sum_{k=0}^{P-1} \rvw_k)}_{\nu_{mean}} \bigg]-\lambda \rvu_x,
$$
We can expand the middle term $\nu_x$ by
$$
\nu_x=\rvv^\top (F^{-1}\Lambda^xF)\rvw=(F^*\rvv){\Lambda^x}(F\rvw)=(1,\omega^x,\cdots,\omega^{x(P-1)})(\overline{F\rvv} \odot F\rvw ),
$$
At the same time, the Fourier transform of $\nu_{mean}$ will be absorbed into the zero-frequency component, because it is independent of the index $x$.

Multiplying on the left by $F$, we have: 
\[
(F\rvu)_r=C(\overline{F\rvv})_r (F\rvw)_r-\lambda \rvu_r, \quad (r\ne 0),
\]
Similarly, we find that for each frequency $r\ne 0$,
\[
\dot{\widehat \rvu}_r\approx C\widehat \rvw_r\overline{\widehat \rvv_r}-\lambda\widehat \rvu_r,\qquad
\dot{\widehat \rvv}_r\approx C\widehat \rvw_r\overline{\widehat \rvu_r}-\lambda\widehat \rvv_r,\qquad
\dot{\widehat \rvw}_r\approx C\widehat \rvu_r\widehat \rvv_r-\lambda\widehat \rvw_r.
\]

where $\hat{\rvv}$ means DFT and $\bar{\rvv}$ means conjugate. That is, the system is \textbf{decoupled at each nonzero frequency}. Each 3-variable system $\hat \rvu_r,\hat \rvv_r,\hat \rvw_r$ corresponding to frequency $r$ will evolve independently.

The next theorem indicates that the conditions for neurons to grow depend on the initialization, and weight decay plays a filtering role. Moreover, growing neurons always develop a dominant frequency and adhere to specific phase patterns.

\begin{proposition}
 [Dominant frequency and phase alignment]\label{prop:freq}
Fix a hidden unit $d$. For each nonzero frequency $r$, let $U_r(t)=\widehat U_{d,r}(t)$, $V_r(t)=\widehat V_{d,r}(t)$, and $W_r(t)=\widehat {W^d}_{r}(t)$. Assume $C>0$ and that the three mode amplitudes remain nonzero up to the first singularity. Let $s_r^\star$ be the blow-up time of the corresponding undamped system. If there exists some $r$, $s_r^\star<C/\lambda$ and
$s_r^\star=\min_{i\neq 0} \{s_i\},$
then $r$ is the first frequency to blow up. Hence it makes the dominant frequency. Moreover, it is going to hold $\arg U_r(t)+\arg V_r(t)-\arg W_r(t)\to 0$, when $t$ approaches the blow-up time.
\end{proposition}

Through the proposition, we can summarize the mechanism of weight decay filtering. The argument $s^\star_r$ depends on the initial weights of a group of neurons $\rvu_d, \rvv_d, \rvw_d$, and weight decay provides a unified survival threshold, and the competition of patterns and the initial random asymmetry determine which frequencies cross this threshold.

\subsection{Estimation of noise}
The actual training dataset is $80\%$ of the complete data with $\alpha$ noise inside. The following theorem will provide an approximation to bridge the gap caused by sampling.

\begin{theorem}[Estimation of noise]\label{thm:convolution}
Under Proposition~\ref{prop:gradient_flow}, sample $\mathcal D$ as above experiments with training data ratio $\rho$ and uniform-label noise ratio $\alpha$. Take $U^x_\star:= \operatorname{diag}(V\Phi^{+}_xW)-\overline{V}\odot \overline{W}$ as the main term under the full-addition signature. Then we estimate
\[
\dot{U}^x-\frac{\eta c_2(1-\alpha)}{P} U^x_\star+\lambda U^x=  \mathcal E_U^x+\delta,
\]
where $\lVert\delta\rVert _\infty\lesssim \epsilon^3$. The sampling error is controlled by
\[
\mathbb E_{\mathcal D}\mathcal E_\bullet=0, \quad
\operatorname{Var}_{\mathcal D}(\mathcal E_\bullet)
\lesssim\frac{[(1-\rho)+\alpha]\eta^2}{\rho P^3}
\left(c_1^2\epsilon^2+c_2^2\epsilon^4\right).
\]
Similar results also hold for $V$ and $W$.
\end{theorem}

\section{Validation on Fourier frequency}

\begin{figure}[t]
  \centering
  \includegraphics[width=\linewidth]{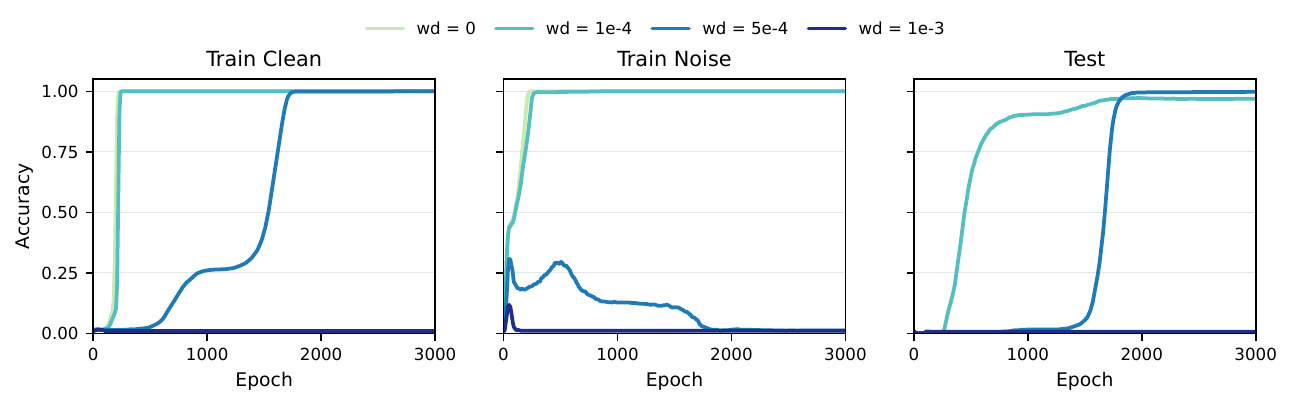}
  \caption{Accuracy-Epoch curves of $\lambda=0,10^{-4},5\times 10^{-4},10^{-3}$. For $\lambda=0$, model reaches 100\% on clean and noise, but 0 on test, overfitting. For $\lambda=10^{-4}$, model reaches almost 100\% on clean, noise and test, generalization. For $\lambda=5\times 10^{-4}$, model forgets the noise. And for $\lambda=10^{-3}$, model cannot converge.}
  \label{fig:curve}
\end{figure}

We regard all the column vectors of $W$ or the row vectors of $U$ and $V$ as \textbf{neurons}, since the output of the neural network is $f(a,b)=\sum_{d=1}^D W^d (U_d e_a + V_d e_b)$. In analysis, we assume that they evolve approximately independently. We denote the concatenation of the $d$-th row of $U,V$, the $d$-th column of $W$ by $\theta_d=(U_d,V_d,W^d)$ and the collection of $\theta_d$ by $\mathcal{V}$. We consider Fourier transform of some neuron $u\in \mathbb{R}^P$, it is
\(
\widehat{u}_k=\sum_{x=0}^{P-1}u_x\,\mathrm{e}^{-\mathrm i2\pi kx/P}\) for \( k=0,\ldots,P-1.
\)
 We define the dominant frequency and sparsity by
\[
\tau(u):=\operatorname{argmax}_{k\neq 0}|\widehat{u}_k|,\quad\rho(u)=\frac{2|\widehat u_{\tau(u)}|^2}{\sum_{k=1}^{P-1}|\widehat u_k|^2}\in[0,1].
\]
It is noted that frequency $r$ and $P-r$ are regarded as the same under Fourier transform. 

\textbf{Experiments.} 
We conducted experiments under different weight decay settings, with the same $\sigma(x)=\text{GELU}(x)$, noise ratio $\alpha = 0.1$ and initialization scale $\gamma = 0.5$, in Figure~\ref{fig:curve}. Additional experiments with other activation functions are reported in Appendix \ref{app:activation},\ref{app:relu}. It is clearly evident from Figure~\ref{fig:curve} that weight decay has a significant impact on the direction of the training process. We observe that $\lambda = 0.0005$ acts as a \emph{regularization}. That is, the model generalizes well while ignoring all the noise. 

We naturally derived from weight decay to consider observing the norms of each group of neurons in the network. 
The growth of the norms can be described as:
\[
\frac12\frac{\mathrm d}{\mathrm dt}\|\theta_d\|^2
=\underbrace{-\eta\langle\theta_d,\nabla_dL\rangle}_{\text{CE drive}}
-\underbrace{\lambda\|\theta_d\|^2}_{\text{decay}}.
\]

\begin{figure}[t]
  \centering
  \includegraphics[width=\linewidth]{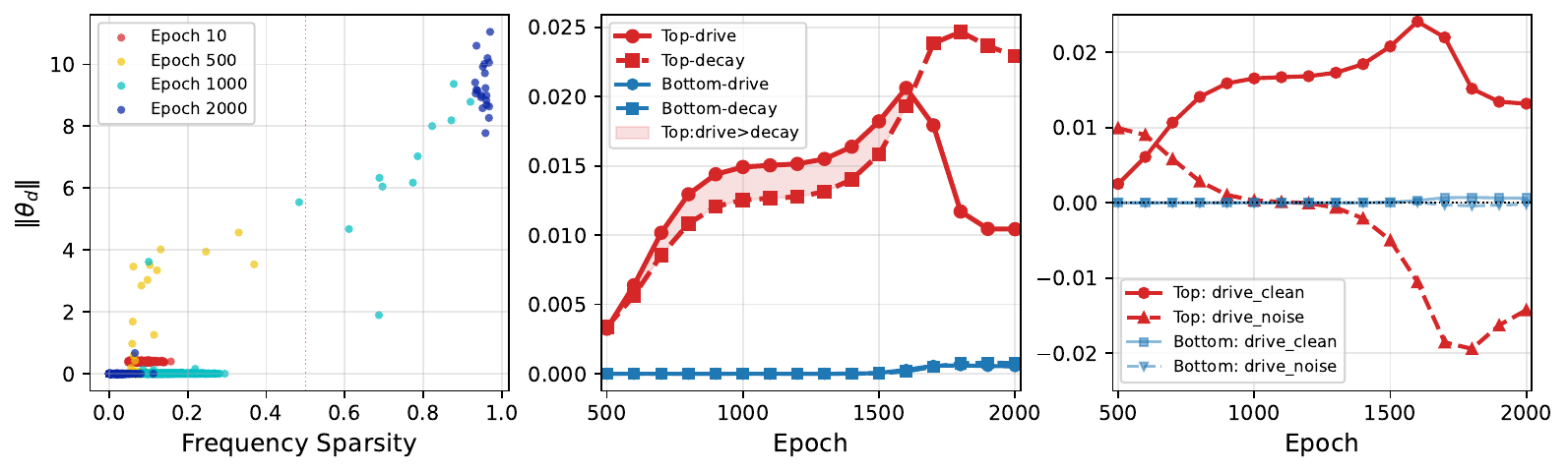}
  \caption{(A) Scatter between Fourier sparsity and the norm of neurons at epoch 10, 500, 1000, 2000. (B) The highest and lowest sparsity 100 neurons are selected, denoted by $\mathcal{V}_t$ and $\mathcal{V}_b$. We observe $\ell_{\mathrm{drive}}^{\mathrm{top}}=\sum_{\theta\in\mathcal{V}_t} \overline{ \ell_{\text{drive}}} $ and similarly for \(\ell_{\mathrm{decay}}^{\mathrm{top}},\ell_{\mathrm{drive}}^{\mathrm{bot}},\ell_{\mathrm{decay}}^{\mathrm{bot}}\)  (C) The clean/noise dataset contributions to the drive. It is $\ell_{\mathrm{drive}}^{\mathrm{top,clean}}=\sum_{\theta\in\mathcal{V}_t} \overline{ \ell_{\text{drive}}} \mid_{\mathcal D_{\text{clean}}} $ and similarly for others. }
  \label{fig:drive_decay_selectivity}
\end{figure}

In Figure ~\ref{fig:drive_decay_selectivity}A, we present the evolution of the norm and frequency sparsity of the neurons. It can be observed that due to weight decay, the norm of most neurons is continuously reduced to 0, while a few neurons can grow. Meanwhile, all the new-grown neurons have relatively high sparsity.

In Figure ~\ref{fig:drive_decay_selectivity}B, we report that norm growth is concentrated on high-sparsity neurons, and the growth of the norm and the improvement of the accuracy occur simultaneously in terms of time, according to Figure \ref{fig:curve}.
We also compare separately on clean and noisy examples in Figure \ref{fig:drive_decay_selectivity}C. It shows that CE--drive term of high-sparsity neurons is contributed primarily by clean data. And these driving forces are almost completely unrelated to the noise in the data.

\paragraph{Cosine Similarity Check.} Theorem \ref{thm:dominance} and \ref{thm:convolution} provide two estimates for the real gradient flow. We conducted numerical experiments in Appendix \ref{app:gradient_flow_approximation}, comparing the cosine similarity between the estimated gradient and the actual gradient. The results show that it can achieve cosine-similarity $\ge0.9$ with the actual gradient.

We find that the dominant frequencies of $U_d$, $V_d$, and $W^d$ tend to agree if the neuron reaches high sparsity. If the common frequency is $\tau$, we extract the components at the dominant frequency by
\((\rvu_d,\rvv_d,\rvw_d)=((\widehat{U_d})_\tau,(\widehat{V_d})_\tau,(\widehat{W^d})_\tau),\) then their phases will tend to satisfy
\(
\arg(\rvu_d)+\arg(\rvv_d)=\arg(\rvw_d) \pmod{2\pi}.
\)
Figure~\ref{fig:frequency_phase} shows that frequency matching and phase alignment emerge together during training, which is consistent with the result of Prop \ref{prop:freq}.

\paragraph{Full-batch GD check.}
Using the full addition table, a GELU-activated network, and shared initialization, we compare exact CE training with the truncated model. We simulate both systems over 50 independently evolved neurons and report the dominant frequencies to which they ultimately converge, as well as the frequency sparsity. For the valid neurons, the dominant frequencies predicted by the truncated model exhibit 93\% agreement with the converged dominant frequencies of the actual GELU-network. Experimental details are provided in the Appendix~\ref{app:frequency}.

\begin{figure}[t]
  \centering
  \includegraphics[width=\linewidth]{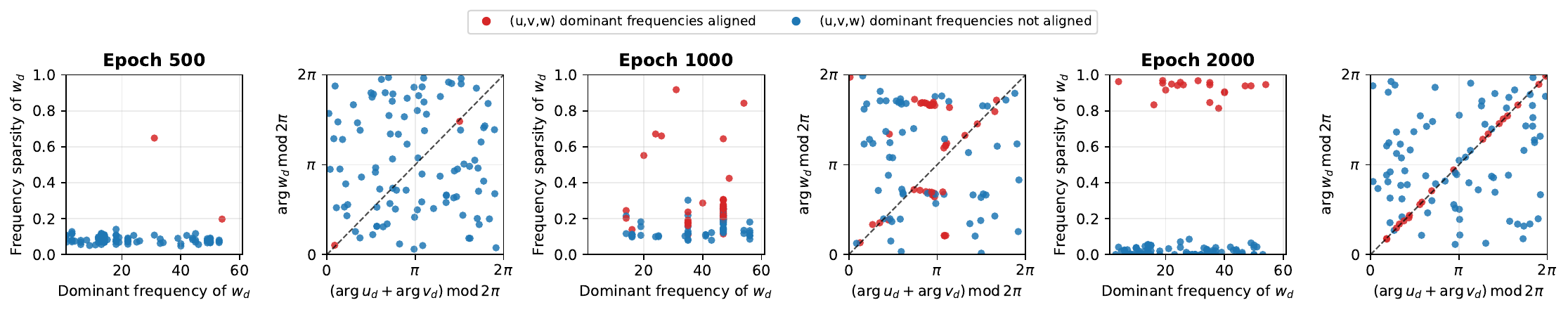}
  \caption{Frequency matching and phase alignment. In the odd columns, the horizontal coordinate is the dominant nonzero Fourier frequency of $\rvw$ and the vertical coordinate is its frequency sparsity; in the even columns, the coordinates are $\arg(\rvu)+\arg(\rvv)\bmod 2\pi$ and $\arg(\rvw)\bmod 2\pi$, respectively. Red points indicate exact agreement of the dominant frequencies of $\rvu_d$, $\rvv_d$, and $\rvw_d$, while blue points indicate a mismatch; dashed lines show phase equality.}
  \label{fig:frequency_phase}
\end{figure}

\section{Application of Probability Signature}
\subsection{The puzzle of early noise}\label{chapter:puzzle}
\begin{figure}[t]
  \centering
  \includegraphics[width=\linewidth]{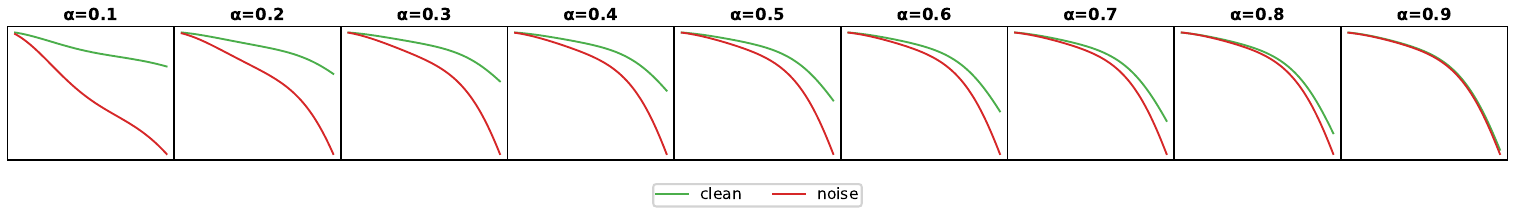}
  \caption{Early($<100$ epochs) clean and noise loss dynamics. $\alpha$ means the noise ratio. Noise loss always decreases more rapidly.}
  \label{fig:early_clean_noise}
\end{figure}
We identify a puzzle in the early training: uniformly random label noise consistently exhibits a faster initial loss decrease than clean data. It was also recorded by \citep{liu2026unveilingmemorizationgeneralizationcoexistencecase} but has not been resolved. From Figure~\ref{fig:early_clean_noise}, this phenomenon holds under almost any ratio of noise. Although this may appear inconsistent with the frequency principle~\citep{xu2020fprinciple}, which states that high-frequency data components are fitted later and low-frequency components earlier. In the discrete setting, the similar redefinition of frequency will be the probability that the predicted label will change while changing one of the input tokens. In this section, we employ probability signatures to quantify this speed difference.

In the small-initialization regime, the loss derivative is approximately
\(
\dot{L}\propto-\sum_{i=1}^{N}\sigma(U^{x_i}+V^{y_i})^\top\left(\sum_{j=1}^{N}\mathbf{1}_{\{z_j=z_i\}}\sigma(U^{x_j}+V^{y_j})\right),
\)
with sampling $(x_i,y_i,z_i)$ from the training distribution. We ignore weight decay because we are comparing the loss of different parts within the same training. From this, we consider counting the numbers of sample pairs that share common coordinates.

For pairs with distinct coordinates define $\kappa(0)=\mathbb{E}[\sigma(\alpha+\beta)^\top\sigma(\gamma+\delta)]$, while for pairs sharing one coordinate define $\kappa(1)=\mathbb{E}[\sigma(\alpha+\beta)^\top\sigma(\alpha+\gamma)]$, where $\alpha,\beta,\gamma,\delta\sim\mathcal{N}(0,1/D)$ are i.i.d. vectors.

Generally, we have $\kappa(1)>\kappa(0)$ (see appendix \ref{app:early_noise_prediction}.1).
Let $n_z$ be the number of examples labeled $z$ and define $K_0=\sum_z n_z^2$. Then the following theorem can compare the early loss for full training dataset. Note that $\varphi^{(a)},\varphi^{(b)}$ here are understood as matrices.

\begin{theorem}[Early loss decrease]\label{thm:early-loss}
Suppose the training dataset contains every distinct pairs $(a,b)\in \Omega \times \Omega$, and the model structure is the same as that mentioned before. For different label distribution, the early-time loss-slope is positively correlated with
\[
\mathbb E \dot{L}_{\text{CE}}\propto-\kappa(0)K_0-(\kappa(1)-\kappa(0))K_1, \quad K_1=P^2(\lVert\varphi^{(a)}\rVert_2^2+\lVert\varphi^{(b)}\rVert_2^2-2).
\]
\end{theorem}
\emph{Proof.} See Appendix~\ref{app:thm-early-loss}.

For modular addition, conditioning on any fixed value of $a$ (or $b$), the label is uniform on $\Omega$. Hence, 
\[
\varphi^{(a)}=\varphi^{(b)}=\frac{1}{P}
\mathbf{1}_P\mathbf{1}_P^\top,
\qquad (K_0,K_1)=(P^3,0),
\]

For a fully random labeling, let $Z_{x,y}=\mathbb{P}_\pi(Z=y\mid X=x)$. The $P$ pairs with first coordinate $x$ receive independent labels, so
\[
Z_{x,y}\sim P^{-1}\operatorname{Binomial}(P,P^{-1}).
\]
According to the binomial distribution, we get 
\(\mathbb{E}[Z_{x,y}^2]=2P^{-2}-P^{-3}.\)
It further gives
\(
\mathbb{E}\left[K_0\right]=P^3+P^2-P,\mathbb{E}\left[K_1\right]=2(P^2-P)\), which are both greater than those in modular-addition case. Thus, uniform random labels can produce a faster early loss decrease than modular-addition labels. 

Next we consider the mixture of clean and noise. Without loss of generality, we assume sampling in complete $\Omega\times \Omega$ and dividing it into clean samples (modular addition) and noise (uniform random label).
The loss for one example $(x_i,y_i,z_i)$ can be regarded as
\[
\dot{L}_i \approx -\kappa(0)n_{z_i}^2-(\kappa(1)-\kappa(0))\Delta_i;\  \ \Delta_i:=\sum_{j=1,\ j\neq i}^{N_{\text{train}}} \mathbf 1_{z_i=z_j}(\mathbf 1_{x_i=x_j}+\mathbf 1_{y_i=y_j}).
\]
Here the coefficient $\Delta_i$ is regarded as $\Delta_{\mathrm{clean}}$ for a clean sample, or $\Delta_{\mathrm{noise}}$ for a noise. Taking expectations, the counts satisfy $\overline{\Delta_{\mathrm{clean}}}=2\alpha(1-P^{-1}),$
and
$\overline{\Delta_{\mathrm{noise}}}=2(1-P^{-1}).$
Consequently, $\overline{\Delta_{\mathrm{noise}}}-\overline{\Delta_{\mathrm{clean}}}=2(1-\alpha)(1-P^{-1})>0,$
so the loss of noise decreases more rapidly.
This prediction is consistent with the experimental observations and can also be applied to other forms of noise. In Appendix \ref{app:multiplication} and \ref{app:bijective}, we conducted additional experiments, thereby verifying that the theoretical predictions presented here are reliable.

\subsection{Extension of the Method: XOR as an example}
\begin{figure}[t]
  \centering
  \includegraphics[width=\linewidth]{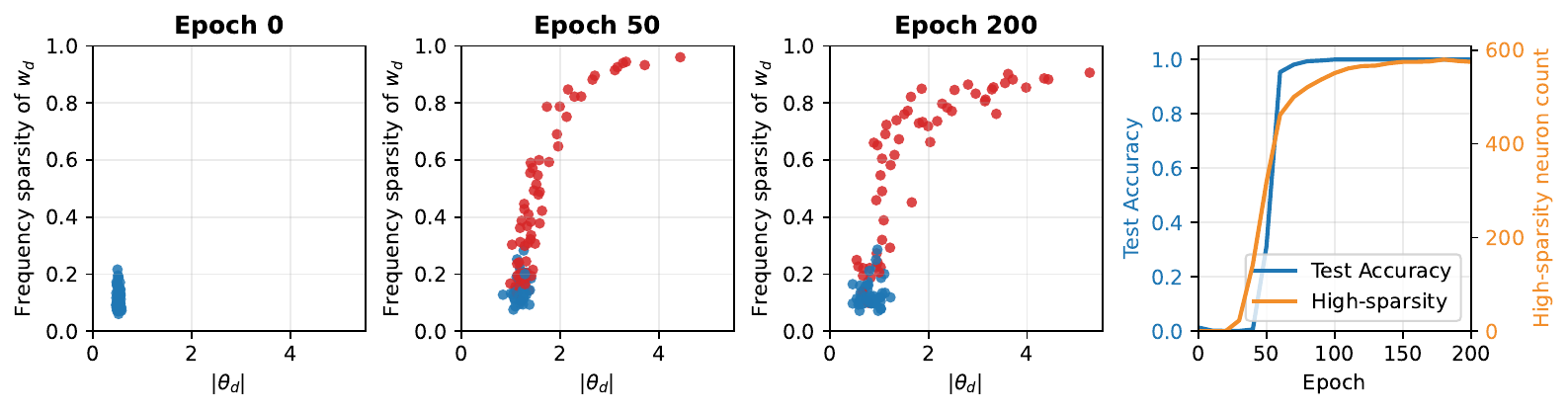}
  \caption{Hadamard frequency structure in the XOR experiment. For sampled neurons at epochs $0$, $50$, $200$, the horizontal axis is $\|\theta_d\|_2$ and the vertical axis is the sparsity of $w_d$. Red points denote neurons whose three dominant frequencies agree. The right panel reports test accuracy and the number of high-sparsity, frequency-matched neurons; high-sparsity means $\rho >0.7$.}
  \label{fig:xor_frequency}
\end{figure}
The Fourier structure observed in modular addition has previously been identified and functionally validated from trained networks. However, these results primarily provide a mechanistic description of the circuit after it has formed, rather than a general forward description to predict the representation basis associated with a new algebraic operation.

Our probability-signature analysis suggests such a prescription. In the small-initialization regime, the leading parameter dynamic is governed by linear operators constructed from conditional statistics of the training distribution. When these probability-signature operators admit a common eigen basis, transforming the parameters into this basis approximately decouples the leading-order dynamics into independent modes. Thus, the relevant representation basis need not be assumed to be the Fourier basis in advance, instead it can be predicted by finding the common eigenvectors of the joint signature.

This observation provides a natural route for transferring the analysis beyond modular addition. Starting from the Fourier transform, we notice that relevant probability signatures can be regarded as the operators on the Abelian group, whose common eigenvectors are given by the corresponding group characters whenever the operators are simultaneously diagonalizable.

For example, consider the task on $\Omega=\{0,\ldots,63\}$, with $F(a,b)=a\operatorname{xor}b$. Its joint signature is
\[
\Phi^{(a)}_x=\sum_{y,z\in \Omega}\mathbb{P}(Z=z,Y=y\mid X=x)e_y\otimes e_z=\frac{1}{64}\sum_{y\in\Omega}\boldsymbol E_{(y,x\operatorname{xor}y)}.
\]
Hadamard basis simultaneously diagonalize these operators: there is an orthogonal basis $H=[H_0,H_1,\ldots,H_{63}]$ such that
$H^{-1}\Phi^{(a)}_x H$
is diagonal for every $x\in \Omega$. The properties of $H$ can be referred to in Appendix~\ref{app:hadamard}. 
Therefore, we claim that after training on XOR, the weights will evolve into a shape similar to the Hadamard basis.

\paragraph{Validation.} We write $u=\sum_{k=0}^{63}E_{u,k}H_k$, and define its sparsity by
\[
\rho_u=\frac{\max_k|E_{u,k}|^2}{\sum_k|E_{u,k}|^2}.
\]
The existence of this orthogonal basis allows the above analysis to be applied to XOR. Based on the redefined frequency sparsity, in Figure~\ref{fig:xor_frequency} we present the observations regarding the frequency matching, as well as the high correlation between frequency sparsity and norm. The observed birth and growth of high-sparsity neurons supports our claim.

\section{Conclusion and Discussion}
Using the probability-signature framework, we provide a unified data-side account of three phenomena. First, for modular addition, the conditional-distribution operators induced by the training data recover the Fourier basis as their common eigenbasis and yield approximately decoupled leading-order dynamics, thereby accounting for the emergence of dominant frequencies, frequency matching, and phase alignment. Second, under uniformly random label noise, conditional label collisions among examples sharing an input coordinate explain why corrupted examples can exhibit a faster initial loss decrease than clean examples. Third, the same operator viewpoint suggests a route for predicting the special structures that neurons may develop on datasets generated by other binary operations: for XOR, the joint signature identifies the Hadamard basis as the relevant candidate coordinate system, in agreement with the weight patterns observed after training. \textbf{Limitations.} We approximate the empirical gradient flow using idealized dynamics derived from discrete, full-distribution probability signatures, together with finite-sample and label-noise corrections. Moreover, our dynamical analysis concerns the dominant gradient terms within a particular training regime; we do not claim that this approximation remains valid throughout the entire training process. Further discussion of the scope and empirical accuracy of this approximation is provided in Appendix~\ref{app:gradient_flow_approximation}.

\bibliography{reference}
\bibliographystyle{reference}

\newpage
\appendix

\section{Model and training details}\label{app:train}

\begin{figure}[t]
  \centering
  \includegraphics[width=\linewidth]{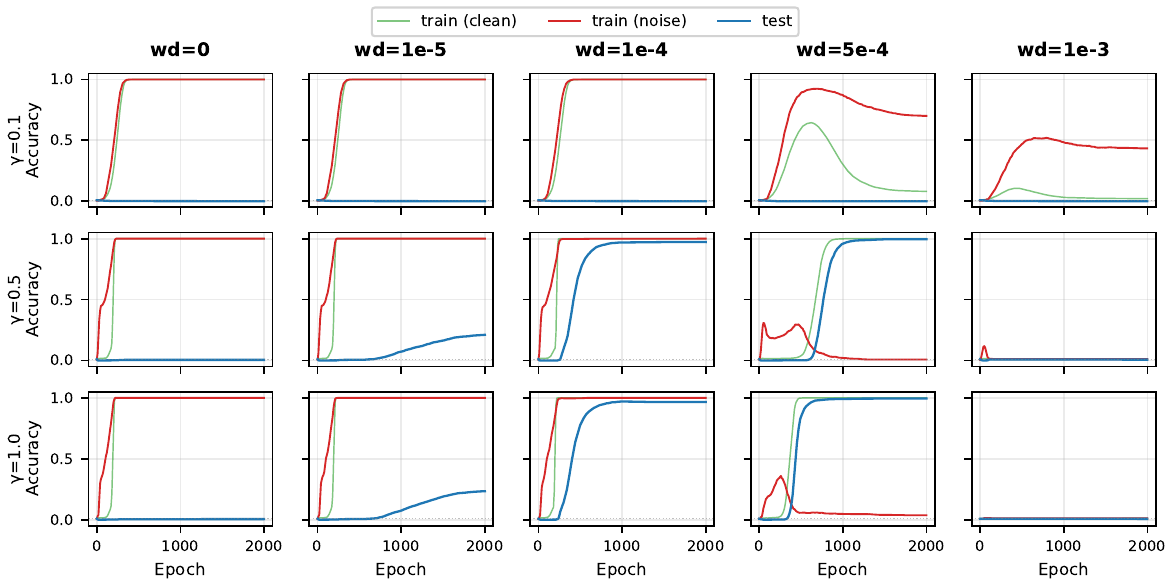}
  \caption{Accuracy curves using Adam, with weight decay $\lambda=0,10^{-5},5\cdot 10^{-4},10^{-4},10^{-3}$ and initialization scale $\gamma=0.1,0.5,1.0$. }
  \label{fig:acc_adam}
\end{figure}

\begin{figure}[t]
  \centering
  \includegraphics[width=\linewidth]{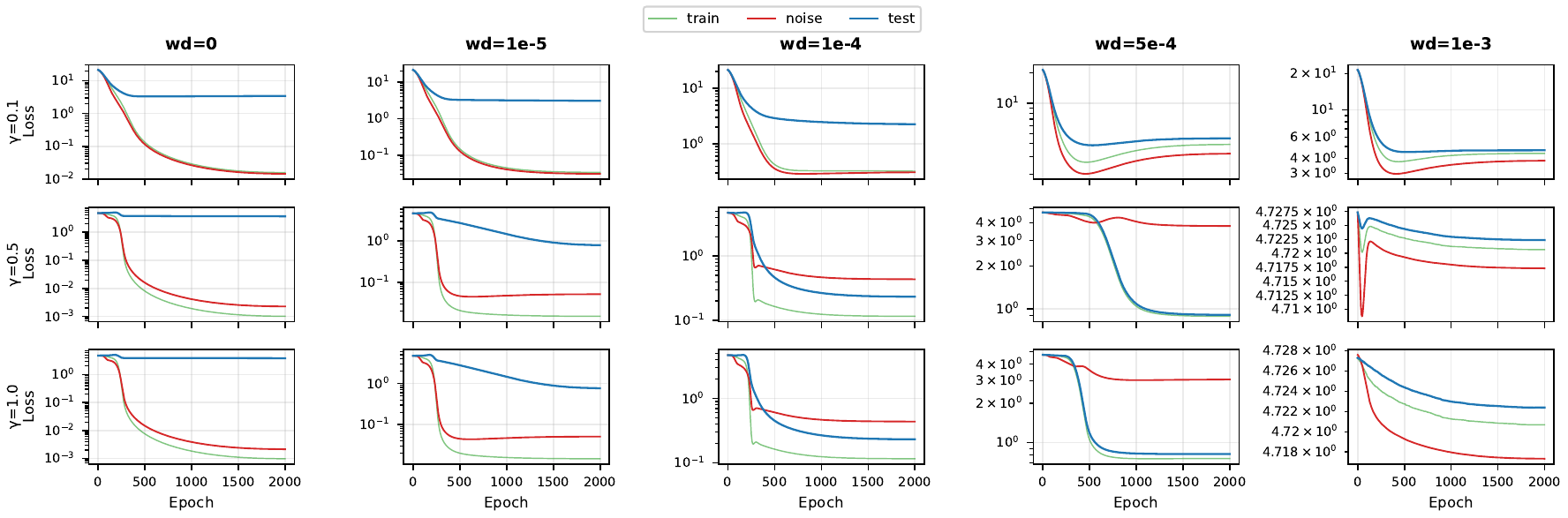}
  \caption{Loss curves using Adam.}
  \label{fig:loss_adam}
\end{figure}

The neural network structure of this article can be written as

$$
f(a, b) = W\sigma(U e_a + V e_b)=\sum_{d=0}^{D-1}W^d(U_{d,a}+V_{d,b}),
$$

Here, $U$ and $V$ are input embedding matrices in $\mathbb{R}^{D \times P}$, and $W$ is the output matrix in $\mathbb{R}^{P \times D}$, with $D$ representing the width of the hidden layer ($D=2000$, $P=113$). The output probability can be expressed as

$$
p(X_i)_{z_i} = \frac{\exp(f(a, b)_{z_i})}{\sum_{k=0}^{P-1} \exp(f(a, b)_k)}
$$

The loss of the training data is given by the cross-entropy:

$$
L_{\text{CE}} = \frac{1}{N}\sum_{i=1}^N -\log p(X_i)_{z_i}
$$

The total loss after adding weight decay is:

$$
L = L_{\text{CE}} + \frac{\lambda}{2}\left(\|U\|_F^2 + \|V\|_F^2 + \|W\|_F^2\right)
$$

The learning rate strategy we employ is warmup + cosine decay:

$$
\eta_t = \begin{cases}
\eta \cdot \frac{t}{T_{\text{warmup}}}, & t < T_{\text{warmup}} \\
\eta \cdot \frac{1}{2}\left(1+\eta_0 + (1-\eta_0)\cos\frac{(t - T_{\text{warmup}})\pi}{T_{\text{optim}} - T_{\text{warmup}}}\right), & t \geq T_{\text{warmup}}
\end{cases}
$$

\begin{figure}[t]
  \centering
  \includegraphics[width=\linewidth]{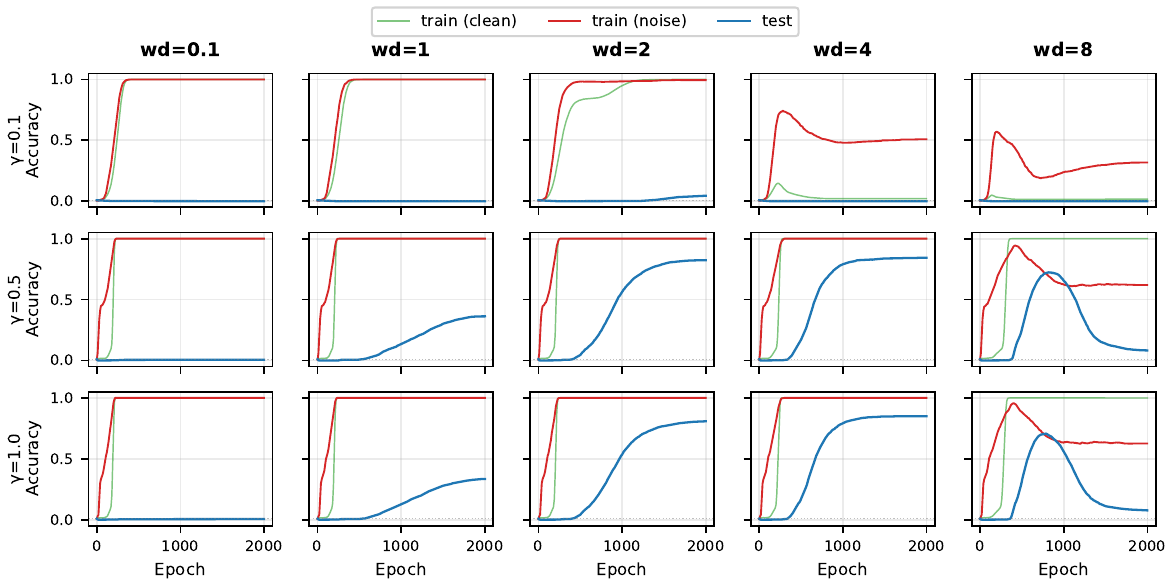}
  \caption{Accuracy curves using AdamW.}
  \label{fig:acc_adamw}
\end{figure}

\begin{figure}[t]
  \centering
  \includegraphics[width=\linewidth]{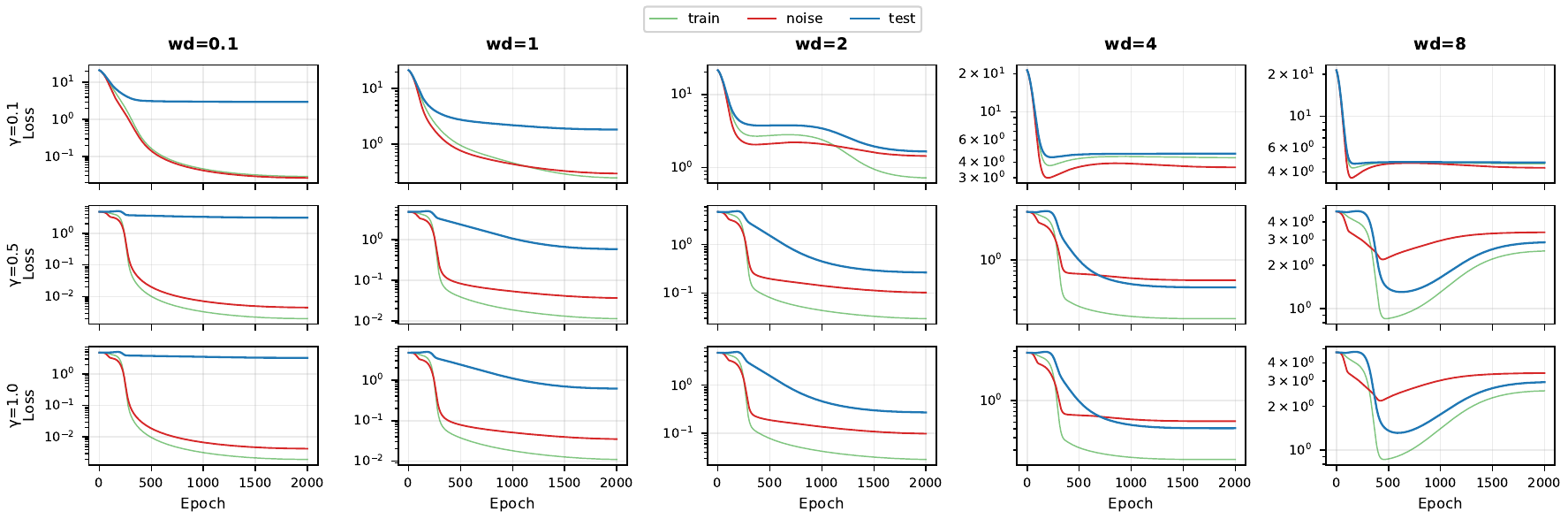}
  \caption{Loss curves using AdamW.}
  \label{fig:loss_adamw}
\end{figure}

The minimum learning rate $\eta_0 = 10^{-5}$, and the initial learning rate $\eta = 10^{-3}$ if using the Adam/AdamW optimizer. Experiments using Adam can generally be reproduced under the AdamW optimizer with about 1000x weight decay. Figure \ref{fig:acc_adam},\ref{fig:loss_adam} compares the specific accuracy or loss curves with Adam under different weight decay and initialization. If the weight decay is too small or absent, there will be no generalization. If the weight decay is too large, it may cause the training to fail to converge, and at this point, all parameters will be compressed to zero. With an appropriate weight decay, using a smaller initial value can shorten the training convergence time. Moreover, check Figure \ref{fig:acc_adamw},\ref{fig:loss_adamw} for AdamW. Compared with Adam, greater weight decay needs to be used to achieve a similar effect, and we find that in this problem the Adam optimizer can provide better regularization, that is, a reduction of noise.

\section{Discussion on Activation Functions}\label{app:activation}

\begin{figure}[t]
  \centering
  \includegraphics[width=\linewidth]{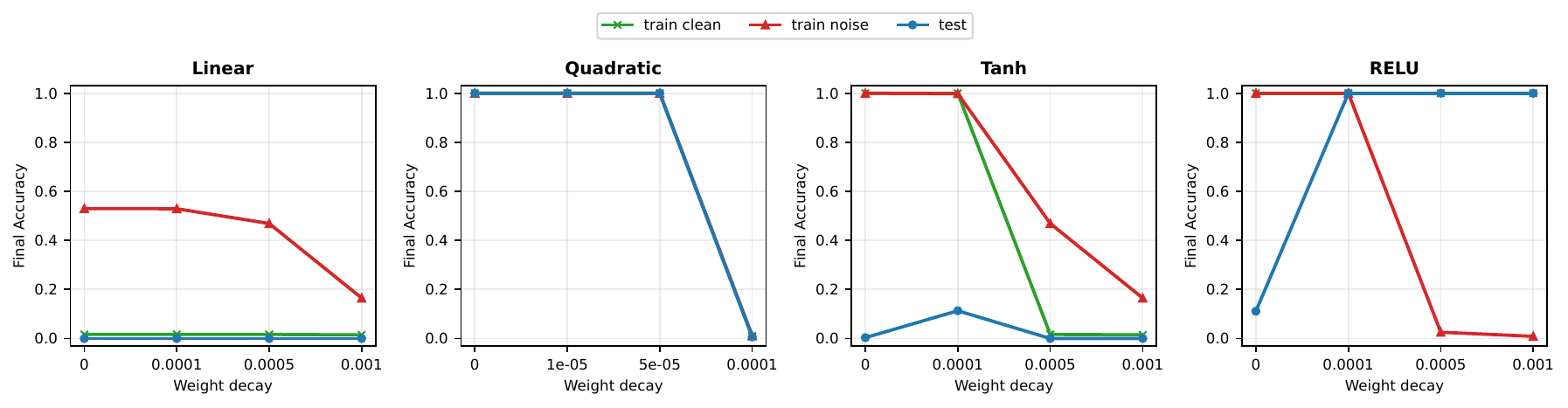}
  \caption{(a)--(d): Final accuracy using activation functions $\sigma(x)=x$ (linear), $\sigma(x)=x^2$ (quadratic), $\sigma(x)=\tanh(x)$, or $\sigma(x)=\max(0,x)$, under different weight decay. }
  \label{fig:activation}
\end{figure}

We use $\sigma(x)=\operatorname{GELU}(x)$ as the activation function in the main text, which is smooth near zero. The analytical method in our text is based on the expansion of the activation function, hence we conducted the same experiments with other activation functions, such as $\sigma(x)=x$ (linear), $\sigma(x)=x^2$, $\sigma(x)=\tanh(x)$, and $\sigma(x)=\operatorname{ReLU}(x)$. The experimental results are presented in Figure \ref{fig:activation}. From the figure, it can be seen that using the linear activation is almost impossible to learn well on the training set. Using $\tanh(x)$ as the activation function results in very low test generalization. However, using the quadratic or ReLU activation function can achieve great test generalization under an appropriate weight decay. 

Let's review the mathematical derivation of the probability signature presented in the main text, we use the Taylor expansion of the activation function,
\[
\sigma(h)=0+\sigma'(0)h+\frac{1}{2}\sigma''(0)h^2+o(h^2),
\]
to derive the probability signature. We can categorize all the expanded terms into 1-order or 2-order terms, based on whether they come from the 1st-order expansion or the 2nd-order expansion. In Theorem \ref{thm:dominance}, the 1-order terms are eliminated because of the assumption of uniform marginal distribution. As a result, the second-order derivative term is the key to this problem.

The theoretical explanation matches our observation of poor generalization with linear/Tanh. Because $\tanh(x) = x - x^3/3 + \cdots$ does not contain the second-order derivative term. Meanwhile, the quadratic activation function contains the second-order derivative term, which does well in the test case. Besides, the success of RELU is not in contradiction to the analysis, because it is not differentiable near 0. We will discuss the RELU activation function in detail in Appendix \ref{app:relu}.

\section{Experimental Details}
\subsection{Gradient Flow Approximation}
\label{app:gradient_flow_approximation}
In this section, we present the experimental results with the aim of verifying the rationality of the approximation methods employed in Theorem \ref{thm:dominance} and Theorem \ref{thm:convolution}.

\begin{figure}[t]
  \centering
  \includegraphics[width=\linewidth]{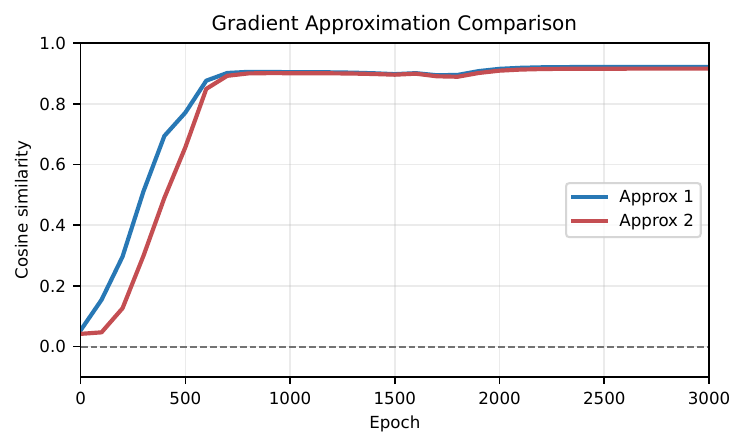}
  \caption{ Cosine similarity between the approximated gradients and the actual gradients. The blue curve uses the main term from Theorem \ref{thm:dominance} and the red one uses the main term from Theorem \ref{thm:convolution}.}
  \label{fig:gradient_flow_approximation}
\end{figure}

By Theorem \ref{thm:dominance}, we can approximate the gradient flow of $U^x$ as $\dot{U^x}\approx -\lambda U^x+c\cdot \operatorname{diag}(V\Phi^{(a)}_x W)$, where $\Phi^{(a)}_x$ is the probability signature of the actual training dataset. We extract its main term, 
\[
\hat\nabla_{U^x}L_1 := \operatorname{diag}(V\Phi^{(a)}_x W),
\]
as an approximation for the cross-entropy gradient. Flatten this vector, we get $\hat\nabla_{\theta}L_1$, which is an approximation for the cross-entropy gradient $\nabla_{\theta}L$. We compare the two gradients by cosine similarity 
\[
\cos \frac{\langle \hat\nabla_{\theta}L_1, \nabla_{\theta}L \rangle}{\|\hat\nabla_{\theta}L_1\|_2 \cdot \| \nabla_{\theta}L \|_2},
\]
as Approx 1 in Figure \ref{fig:gradient_flow_approximation}. 

Similarly, by Theorem \ref{thm:convolution}, we can approximate the gradient flow of $U^x$ as $\dot{U^x}\approx -\lambda U^x+c\cdot \operatorname{diag}(V\Phi^{+}_x W)$, where $\Phi^{+}_x$ is the probability signature of the clean modular addition, it is, 
\[
\Phi^{+}_x = \frac{1}{P} \begin{pmatrix}
 & 1 & & & \\
 & & 1 & & \\
 & & & \ddots & \\
 & & & & 1 \\
 1 & & & & 
\end{pmatrix}^x
\]
We also extract and flatten its main term as $\hat\nabla_{\theta}L_2$, comparing which with the cross-entropy gradient $\nabla_{\theta}L$ by cosine similarity, as Approx 2 in Figure \ref{fig:gradient_flow_approximation}. The experimental results show that both the approximated gradients and the actual gradients have a high similarity of 90\%. This verifies our theoretical analysis. Meanwhile, the control experiments using $\sigma(x)=x^2$ can serve as an ablation study, which demonstrates that using only the 2-order main terms is sufficient for achieving test generalization in the modular-addition task.

\textbf{Discussion on the small initialization assumption.}
In the main text, the approximation of the gradient is required to be expanded within a small parameter region. At the same time, we claim that the neurons we are studying can always grow. This note will clarify that this is not in contradiction.

First of all, we acknowledge that there are higher-order terms in the gradient flow equation such as $\delta=O(h^3)$. Our theoretical derivation does not assume that these higher-order terms should be omitted. We focus on the dominant part. There exists a scale $h_0$. When the parameter is less than $h_0$, the second-order term becomes dominant, while when the parameter is greater than $h_0$, the higher-order terms become dominant. It is noted that when weight decay exists, there is a non-high-order term of $-\lambda \theta$ in the gradient flow equation. Therefore, an appropriate $\lambda$ will position the decay term between the above, and when the scale of the parameters expands beyond the region dominated by the second-order, the weight decay becomes the main term, thereby preventing the scale of the parameters from continuing to grow.

\subsection{Frequency Selectivity }
\label{app:frequency}

\begin{figure}[t]
  \centering
  \includegraphics[width=\linewidth]{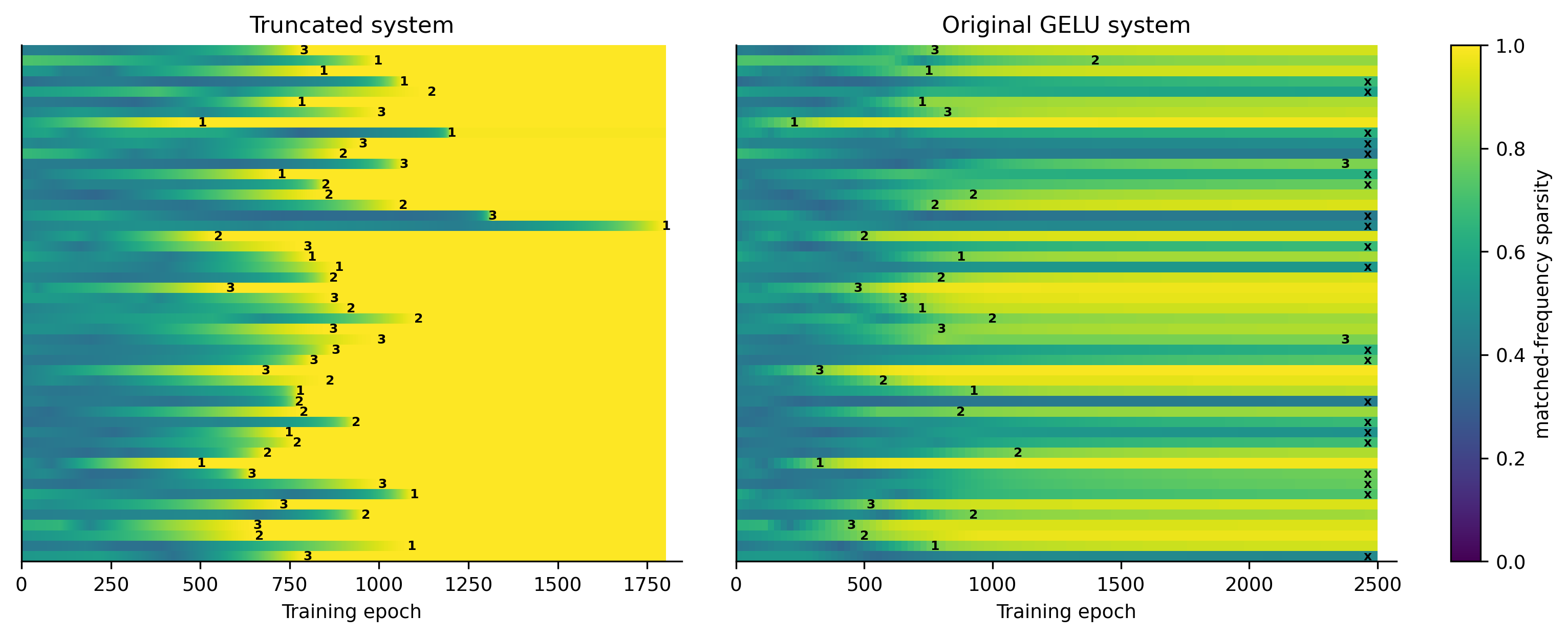}
  \caption{ Neuron-wise frequency concentration in the truncated and original
  GELU systems. Heatmaps show the evolution of matched-frequency sparsity for
  individual neurons during training. Rows correspond to neurons, and columns
  represent training epochs. Numbers indicate the final winning Fourier mode,
  whereas 'x' marks neurons whose winner had not stabilized by the specified
  threshold. Higher values indicate stronger concentration onto a matched
  nonzero frequency. }
  \label{fig:freq_predict}
\end{figure}

We used complete addition tables with $P=7,D=50$, GELU activation, Gaussian scale $D^{-1/2}$. Both models started from identical initialized parameters. The exact update was
\[
\theta_{k+1}=\theta_k+0.2(-\nabla L(\theta_k)-10^{-4}\theta_k),
\]
and the truncated update replaced $-\nabla L$ by the field in Theorem \ref{thm:dominance} with $c=\sigma''(0)/7$ and the same weight decay $\lambda=0.0001$.

Recording every 25 epochs, the experiment tracked the evolution of the dominant frequency of each neuron, in Figure \ref{fig:freq_predict}. We found that approximately 50\% of the neurons developed a dominant frequency, and among this subset, a proportion of 0.93 converged to the same dominant frequency. Because shared initialization parameters were used, these neurons gave rise to the same frequency selection, which is consistent with our argument in Proposition \ref{prop:freq}.

\section{Discussion on the ReLU activation}\label{app:relu}
In this section, we will discuss the ReLU activation function in detail. The ReLU activation function is not differentiable at 0, which makes it difficult to analyze using Taylor expansion. However, we can still analyze the behavior of the network with ReLU activation by considering the piecewise linear nature of ReLU.

\subsection{Fourier-mode is able to do modular addition under RELU}
We give a note on how 2L-MLP with ReLU activation do well in modular addition. 
\begin{proposition}[A sufficient Fourier condition for modular addition]
Let $P$ be an odd prime, and let $w_d\in\mathbb Z_P\setminus\{0\},\, 0\le d<D$ be prescribed frequencies. For each hidden unit $d$, let
$\phi_{1,d},\phi_{2,d}$ be mutually independent random variables
distributed uniformly on $[0,2\pi)$. If the parameters are set to
$
U_{d,x}
=
\cos\!\left(\frac{2\pi w_d x}{P}+\phi_{1,d}\right),
\,
V_{d,x}
=
\cos\!\left(\frac{2\pi w_d x}{P}+\phi_{2,d}\right),
\,
W_{c,d}
=
\cos\!\left(
\frac{2\pi w_d c}{P}
+\phi_{1,d}+\phi_{2,d}
\right).
$
then the prediction of the neural network is $(a+b) \mod P$ with probability tending to 1 as $D\to\infty$.
\end{proposition}
\begin{proof}
Fix $a,b\in\mathbb Z_P$, and write
\[
c_\star=(a+b)\bmod P.
\]
For each hidden unit $d$ and output coordinate $c$, define
\[
X_{d,c}
:=
W_{c,d}\max\{0,U_{d,a}+V_{d,b}\}.
\]
Then
\[
f_c^{(D)}(a,b)=\frac1D\sum_{d=1}^D X_{d,c}.
\]
Since $|W_{c,d}|\leq 1$ and
\[
0\leq \max\{0,U_{d,a}+V_{d,b}\}\leq 2,
\]
we have
\[
-2\leq X_{d,c}\leq 2.
\]

We first compute the expectation of a single-unit contribution. Set
\[
\alpha_{d}
=
\frac{2\pi w_d a}{P}+\phi_{1,d},
\qquad
\beta_{d}
=
\frac{2\pi w_d b}{P}+\phi_{2,d},
\]
and let
\[
\delta_{d,c}
=
\frac{2\pi w_d(c-a-b)}{P}.
\]
Because
\[
\phi_{1,d}+\phi_{2,d}
=
\alpha_d+\beta_d-\frac{2\pi w_d(a+b)}{P},
\]
we obtain
\[
W_{c,d}
=
\cos(\alpha_d+\beta_d+\delta_{d,c}).
\]
Thus,
\[
\mathbb E[X_{d,c}]
=
\mathbb E\left[
\cos(\alpha_d+\beta_d+\delta_{d,c})
\max\{0,\cos\alpha_d+\cos\beta_d\}
\right].
\]

Since $\alpha_d$ and $\beta_d$ are independent and uniformly distributed on
$[0,2\pi)$, we may set
\[
s=\frac{\alpha_d+\beta_d}{2},
\qquad
t=\frac{\alpha_d-\beta_d}{2}.
\]
Using
\[
\cos\alpha_d+\cos\beta_d
=
2\cos s\cos t,
\]
we have
\[
\mathbb E[X_{d,c}]
=
\frac{1}{(2\pi)^2}
\int_0^{2\pi}\int_0^{2\pi}
\cos(2s+\delta_{d,c})
\,[2\cos s\cos t]_+
\,dt\,ds.
\]
For fixed $s$,
\[
\int_0^{2\pi}[2\cos s\cos t]_+\,dt
=
4|\cos s|.
\]
Consequently,
\begin{align*}
\mathbb E[X_{d,c}]
&=
\frac{1}{\pi^2}
\int_0^{2\pi}
|\cos s|\cos(2s+\delta_{d,c})\,ds\\
&=
\frac{\cos(\delta_{d,c})}{\pi^2}
\int_0^{2\pi}|\cos s|\cos(2s)\,ds\\
&=
\frac{4}{3\pi^2}
\cos(\delta_{d,c}).
\end{align*}
In the second equality, the sine term vanishes by symmetry, and in the last equality we used
\[
\int_0^{2\pi}|\cos s|\cos(2s)\,ds=\frac43.
\]
Therefore,
\[
\mathbb E[X_{d,c}]
=
\frac{4}{3\pi^2}
\cos\left(
\frac{2\pi w_d(c-a-b)}{P}
\right).
\]

For the target coordinate $c_\star$, we have
\[
c_\star-a-b\equiv 0\pmod P,
\]
and hence
\[
\mathbb E[X_{d,c_\star}]
=
\frac{4}{3\pi^2}.
\]
Now consider any $c\neq c_\star$. Since $P$ is prime,
\[
w_d\not\equiv 0\pmod P,
\qquad
c-a-b\not\equiv 0\pmod P
\]
imply
\[
w_d(c-a-b)\not\equiv 0\pmod P.
\]
Define
\[
\rho_P
:=
\max_{r\in\mathbb Z_P\setminus\{0\}}
\cos\left(\frac{2\pi r}{P}\right).
\]
Because $\mathbb Z_P\setminus\{0\}$ is finite and none of the corresponding angles is congruent to $0$ modulo $2\pi$,
\[
\rho_P<1.
\]
It follows that, for every $d$ and every $c\neq c_\star$,
\[
\mathbb E[X_{d,c}]
\leq
\frac{4}{3\pi^2}\rho_P.
\]
Hence, with
\[
\Delta_P
:=
\frac{4}{3\pi^2}(1-\rho_P)>0,
\]
we have
\[
\mathbb E[X_{d,c_\star}]
-
\mathbb E[X_{d,c}]
\geq \Delta_P
\qquad
\text{for all }d
\text{ and all }c\neq c_\star.
\]
Averaging over $d$ gives
\[
\mathbb E[f_{c_\star}^{(D)}(a,b)]
-
\mathbb E[f_c^{(D)}(a,b)]
\geq \Delta_P
\qquad
\text{for every }c\neq c_\star.
\]

It remains to control the fluctuations around these expectations. Since the random phases are mutually independent across hidden units, the variables
$\{X_{d,c}\}_{d=1}^D$ are independent for each fixed $c$. Moreover,
$X_{d,c}\in[-2,2]$. Therefore, Hoeffding's inequality yields, for every
$\varepsilon>0$,
\[
\mathbb P\left(
\left|
f_c^{(D)}(a,b)-\mathbb E[f_c^{(D)}(a,b)]
\right|
\geq \varepsilon
\right)
\leq
2\exp\left(-\frac{D\varepsilon^2}{8}\right).
\]
Taking $\varepsilon=\Delta_P/3$ and applying the union bound over the
$P$ output coordinates, we obtain
\[
\mathbb P\left(
\max_{c\in\mathbb Z_P}
\left|
f_c^{(D)}(a,b)-\mathbb E[f_c^{(D)}(a,b)]
\right|
\geq \frac{\Delta_P}{3}
\right)
\leq
2P\exp\left(-\frac{D\Delta_P^2}{72}\right).
\]
On the complementary event, for every $c\neq c_\star$,
\begin{align*}
f_{c_\star}^{(D)}(a,b)-f_c^{(D)}(a,b)
&\geq
\left(
\mathbb E[f_{c_\star}^{(D)}(a,b)]
-
\mathbb E[f_c^{(D)}(a,b)]
\right)
-\frac{2\Delta_P}{3}\\
&\geq
\frac{\Delta_P}{3}>0.
\end{align*}
Thus, on this event, $c_\star$ is the unique maximizing output coordinate. Consequently,
\[
\mathbb P\left(
\arg\max_{c\in\mathbb Z_P}
f_c^{(D)}(a,b)
=
c_\star
\right)
\geq
1-
2P\exp\left(-\frac{D\Delta_P^2}{72}\right).
\]
Since the right-hand side converges to $1$ as $D\to\infty$, we conclude that
\[
\mathbb P\!\left(
\arg\max_{c\in\mathbb Z_P}f_c^{(D)}(a,b)
=
(a+b)\bmod P
\right)
\longrightarrow 1.
\]
\end{proof}

\subsection{Experiment Details of ReLU activation}
We conducted experiments using the ReLU activation function, and the results are shown in Figures \ref{fig:drive_decay_selectivity_relu} and \ref{fig:frequency_phase_relu}. The experimental setup is consistent with that of the GELU activation function. The results indicate that the ReLU activation function exhibits similar behavior to GELU in terms of weight decay selectivity or frequency-phase scatter evolution. This suggests that the observed phenomena are not specific to those smooth activation function but rather reflect more general properties of neural network training dynamics.

\begin{figure}[t]
  \centering
  \includegraphics[width=\linewidth]{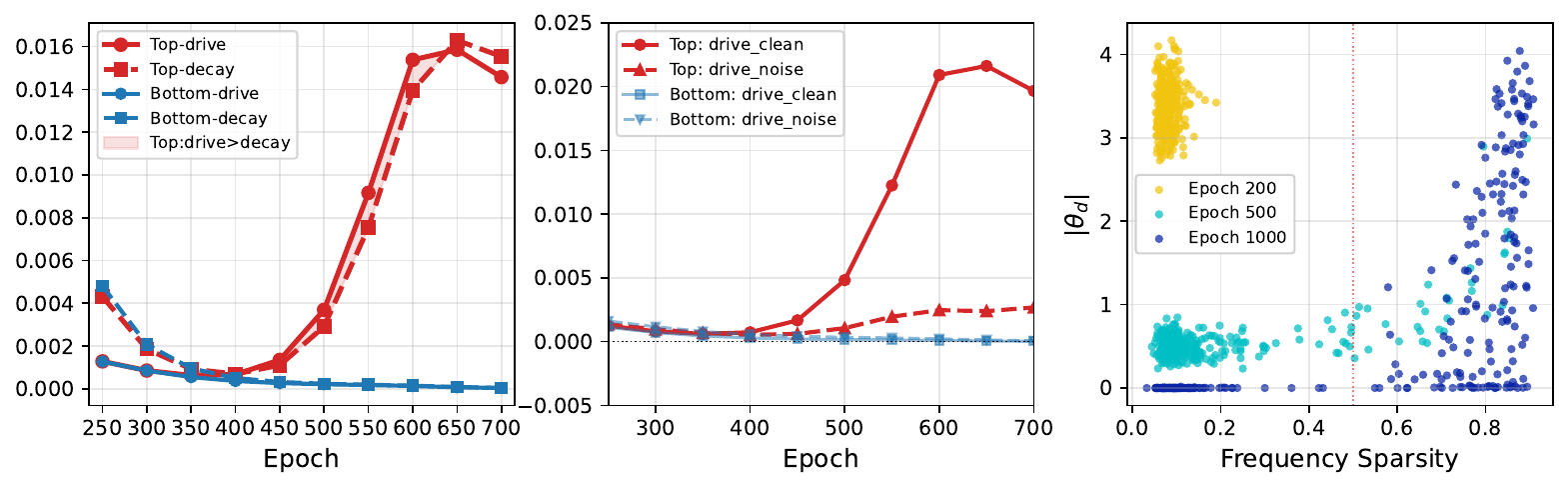}
  \caption{ Experiment of Figure \ref{fig:drive_decay_selectivity} with ReLU. }
  \label{fig:drive_decay_selectivity_relu}
\end{figure}

\begin{figure}[t]
  \centering
  \includegraphics[width=\linewidth]{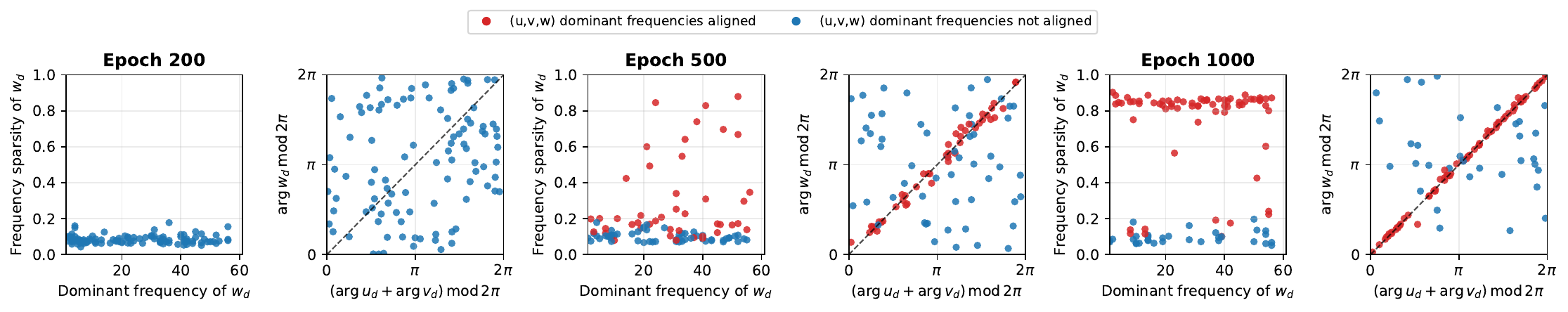}
  \caption{ Experiment of Figure \ref{fig:frequency_phase} with ReLU. }
  \label{fig:frequency_phase_relu}
\end{figure}

\subsection{Early Memorization}

\begin{figure}[t]
  \centering
  \includegraphics[width=\linewidth]{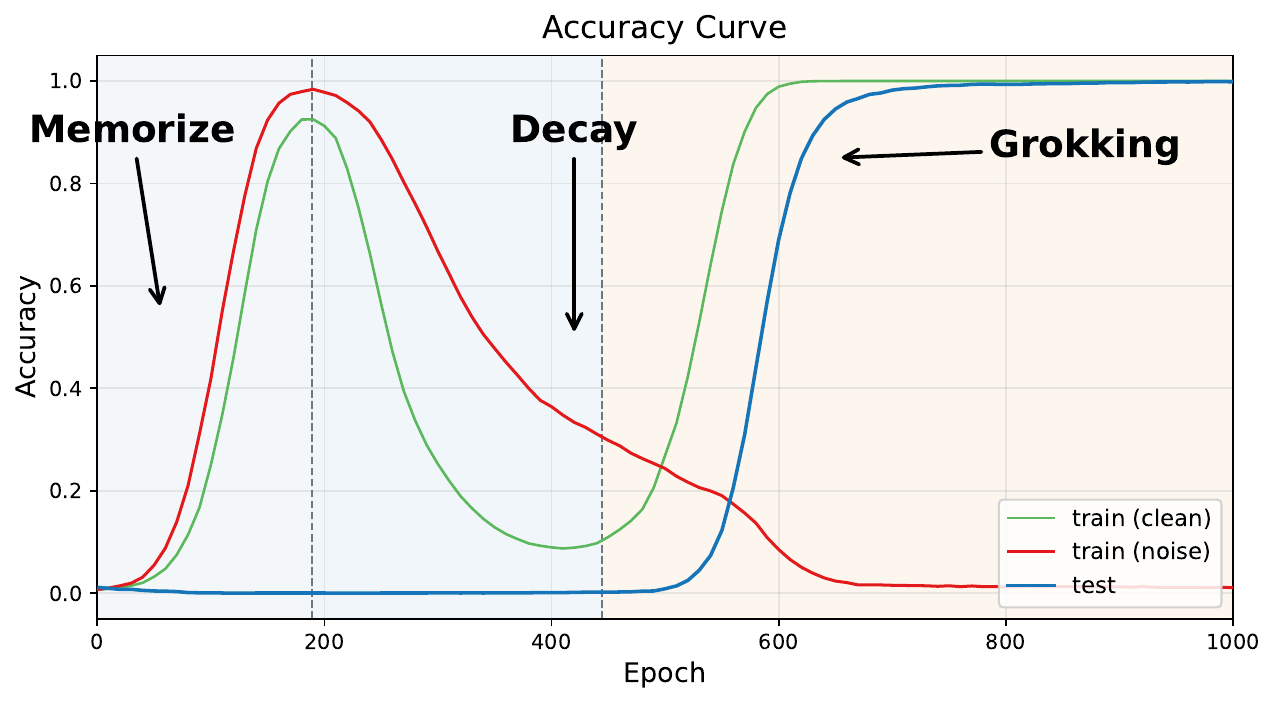}
  \caption{Accuracy curve under RELU activation and $(\lambda,\gamma)=(0.001,0.1)$, with 30\% noise. }
  \label{fig:RELU_curve}
\end{figure}

Current research~\cite{zhang2017understandingdeeplearningrequires, jacot2020ntk} have already provided a thorough discussion on the early memory behavior of neural networks. When use $\sigma(x)=\max(0,x)$, we observe clear traces of memory at $(\lambda,\gamma)=(0.001,0.1)$. The training curve from Figure \ref{fig:RELU_curve} shows it reaches high accuracy on both clean and noise dataset first and then decreases due to the weight decay. Here, we merely use the gradient flow to offer a more intuitive understanding.

Consider a single example $(a,b,y)$ and let $h=Ue_a+Ve_b$ and $p=\operatorname{softmax}(f(a,b))$. By the chain rule,
\[
\frac{\partial L}{\partial W_{\nu,d}}=(p-e_y)_\nu\,\sigma(h)_d+\lambda W_{\nu,d},\qquad
\frac{\partial L}{\partial U_{d,x}}=\mathbf{1}_{\{x=a\wedge h_d>0\}}[W^\top(p-e_y)]_d+\lambda U_{d,x}.
\]
Thus the gradient flow for the output weights is
\[
\frac{\mathrm{d}W_{\nu,d}}{\mathrm{d}t}=(e_y-p)_\nu\sigma(h)_d-\lambda W_{\nu,d}.
\]
Near initialization, the softmax expansion gives
\[
\operatorname{softmax}(f)=\frac{1}{P}\mathbf{1}_P+\frac{1}{P}(f-\overline f)+O(\lVert f\rVert^2/P),
\]
so the correct output coordinate is amplified while the others are weakly suppressed. In particular, the active coordinates satisfy $\frac{\mathrm{d}W_y^\top}{\mathrm{d}t}=\sigma(h)-\lambda W_y^\top,$ 
hence the weights will increase in those dimensions where $\sigma(h) > 0$. For the example in Figure~\ref{fig:interpolation_memscore}(a), $\sigma(x)=\max (0,x)$, only $d=3,5$ are active in the case, namely those satisfying $h_d=U_{d,a}+V_{d,b}>0$. This is only a toy model of 1 training sample. It behaves like an interpolation, storing the data onto these random dimensions.

\begin{figure}[t]
  \centering
  \includegraphics[width=\linewidth]{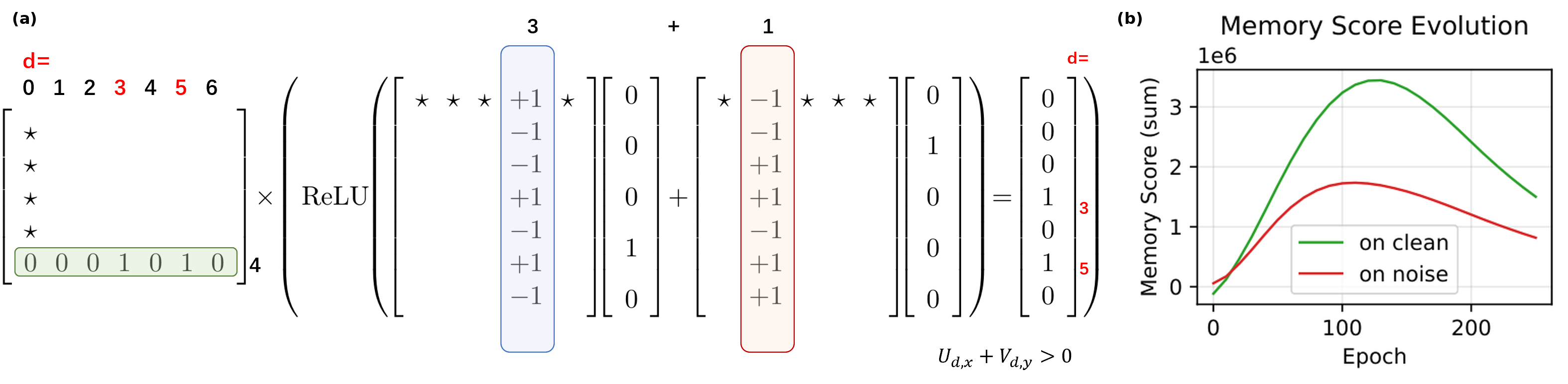}
  \caption{Example-wise interpolation and the memory score. (a). For the illustrative input $3+1=4$ with $(D,P)=(7,5),\sigma(x)=\max(0,x)$, the activation feature $K=e_3+e_5$ represents the initially active hidden coordinates; the displayed $\pm1$ values denote random initialization signs rather than fixed parameter values. (b). Evolution of the memory score on clean and noisy training examples.}
  \label{fig:interpolation_memscore}
\end{figure}

In the full-batch setting, the exact flows are
\[
\frac{\mathrm{d}W_{c,d}}{\mathrm{d}t}=-\frac{1}{N_{\rm train}}\sum_i[p_{i,c}-\mathbf{1}_{\{y_i=c\}}]\max(0,U_{d,a_i}+V_{d,b_i})-\lambda W_{c,d},
\]
\[
\frac{\mathrm{d}U_{d,a}}{\mathrm{d}t}=-\frac{1}{N_{\rm train}}\sum_{i:a_i=a}\mathbf{1}_{\{U_{d,a_i}+V_{d,b_i}>0\}}\sum_r[p_{i,r}-\mathbf{1}_{\{y_i=r\}}]W_{r,d}-\lambda U_{d,a},
\]
\[
\frac{\mathrm{d}V_{d,b}}{\mathrm{d}t}=-\frac{1}{N_{\rm train}}\sum_{i:b_i=b}\mathbf{1}_{\{U_{d,a_i}+V_{d,b_i}>0\}}\sum_r[p_{i,r}-\mathbf{1}_{\{y_i=r\}}]W_{r,d}-\lambda V_{d,b}.
\]
These equations show that each example $(X=(a,b),y)$ strengthens its own output coordinate $(y)$ on its active hidden dimensions and suppresses the other output coordinates $(c\neq y)$, with the latter effect smaller by approximately a factor $1/P$. This example-wise interpolation motivates the following memory score:
\[
\operatorname{mem}\mathrel{:=}\sum_{i=1}^{N}\sum_{d=1}^{D}\mathbf{1}_{\{U_{d,a_i}+V_{d,b_i}>0\}}\sum_{c\neq y_i}(W_{y_i,d}-W_{c,d}).
\]
From figure~\ref{fig:interpolation_memscore}(b), we can observe that the memory score undergoes a rapid accumulation process in the early stages of training. Simultaneously, due to the weight decay's effect of reducing the norm, it gradually decreases.

\section{Proofs and Additional Details}\label{app:proofs}

\subsection{Calculation of the gradient flow}\label{app:grad-flow}
\begin{proof}
Let
\[
h_{x,y}:=Ue_x+Ve_y=U^x+V^y,
\qquad
p_{x,y}:=\operatorname{softmax}\bigl(W\sigma(h_{x,y})\bigr).
\]
Without loss of generality, assume that the learning rate $\eta=1$. 
For the softmax cross-entropy loss associated with one sample
\((x,y,z)\),
\[
\ell(x,y,z)
=
-\log (p_{x,y})_z,
\]
the derivative with respect to the logits
\(f(x,y)=W\sigma(h_{x,y})\) is
\[
\nabla_f\ell=p_{x,y}-e_z.
\]
Therefore, by the chain rule,
\[
\nabla_{h_{x,y}}\ell
=
\operatorname{Diag}\bigl(\sigma'(h_{x,y})\bigr)
W^\top(p_{x,y}-e_z).
\]
Since \(U^x\) contributes to \(h_{x,y}\) only when the first position
of the sample is \(x\), the full-batch gradient flow gives
\[
\frac{\mathrm dU^x}{\mathrm dt}
=
r_x\,
\mathbb E_\pi\left[
\sigma'(h_{x,y})
\odot W^\top(e_z-p_{x,y})
\,\middle|\,x
\right].
\tag{1}
\]

Because \(\sigma\) is \(C^3\) near zero and
\[
\sigma(h)
=
c_1h+\frac{c_2}{2}h^2+o(h^2),
\]
where the powers are understood componentwise, we have
\[
\sigma'(h)
=
c_1\mathbf 1+c_2h+o(\|h\|).
\tag{2}
\]
Moreover,
\[
\sigma(h)=c_1h+O(\|h\|^2).
\tag{3}
\]

Let
\[
\Pi_P
:=
I-\frac1P\mathbf 1_P\mathbf 1_P^\top.
\]
The softmax function satisfies
\[
\operatorname{softmax}(0)=\frac1P\mathbf 1_P,
\]
and its Jacobian at zero is
\[
\begin{aligned}
D\operatorname{softmax}(0)
&=
\operatorname{Diag}\left(\frac1P\mathbf 1_P\right)
-
\left(\frac1P\mathbf 1_P\right)
\left(\frac1P\mathbf 1_P\right)^\top
\\
&=
\frac1P I-\frac1{P^2}\mathbf 1_P\mathbf 1_P^\top
\\
&=
\frac1P\Pi_P.
\end{aligned}
\]
Consequently, for a logit vector \(g\) near zero,
\[
\operatorname{softmax}(g)
=
\frac1P\mathbf 1_P+\frac1P\Pi_Pg+O(\|g\|^2).
\tag{4}
\]

Taking \(g=W\sigma(h_{x,y})\) and using (3), we obtain
\[
g
=
c_1Wh_{x,y}+O(\|h_{x,y}\|^2).
\]
Substituting this expression into (4) gives
\[
p_{x,y}
=
\frac1P\mathbf 1_P
+
\frac{c_1}{P}\Pi_PWh_{x,y}
+
O(\|h_{x,y}\|^2).
\tag{5}
\]
Thus,
\[
\begin{aligned}
W^\top(e_z-p_{x,y})
={}&
W^\top e_z-\frac1P W^\top\mathbf 1_P
\\
&-
\frac{c_1}{P}W^\top\Pi_PWh_{x,y}
+
O(\|h_{x,y}\|^2).
\end{aligned}
\tag{6}
\]

Combining (2) and (6), and retaining all constant and first-order
terms in \(h_{x,y}\), yields
\[
\begin{aligned}
&
\sigma'(h_{x,y})
\odot W^\top(e_z-p_{x,y})
\\
={}&
\left(c_1\mathbf 1+c_2h_{x,y}+o(\|h_{x,y}\|)\right)
\\
&\quad\odot
\left(
W^\top e_z-\frac1P W^\top\mathbf 1_P
-\frac{c_1}{P}W^\top\Pi_PWh_{x,y}
+O(\|h_{x,y}\|^2)
\right)
\\
={}&
c_1W^\top e_z
-\frac{c_1}{P}W^\top\mathbf 1_P
\\
&+
c_2h_{x,y}\odot(W^\top e_z)
-\frac{c_2}{P}h_{x,y}\odot(W^\top\mathbf 1_P)
\\
&-
\frac{c_1^2}{P}W^\top\Pi_PWh_{x,y}
+\delta_{x,y},
\end{aligned}
\tag{7}
\]
where \(\delta_{x,y}\) contains all higher-order terms. In particular,
the product
\[
-\frac{c_1c_2}{P}
h_{x,y}\odot\bigl(W^\top\Pi_PWh_{x,y}\bigr)
\]
is quadratic in \(h_{x,y}\) and is therefore included in
\(\delta_{x,y}\).

We now take the conditional expectation of each term in (7). First,
\[
\begin{aligned}
&\mathbb E_\pi\left[
c_1W^\top e_z-\frac{c_1}{P}W^\top\mathbf 1_P
\,\middle|\,x
\right]
\\
&\qquad=
c_1W^\top
\left(
\mathbb E_\pi[e_z\mid x]
-\frac1P\mathbf 1_P
\right).
\end{aligned}
\tag{8}
\]

Since \(h_{x,y}=U^x+V^y\),
\[
\begin{aligned}
h_{x,y}\odot(W^\top e_z)
={}&
U^x\odot(W^\top e_z)
+
V^y\odot(W^\top e_z).
\end{aligned}
\]
Hence,
\[
\begin{aligned}
&\mathbb E_\pi\left[
h_{x,y}\odot(W^\top e_z)
\,\middle|\,x
\right]
\\
&\qquad=
U^x\odot
\left(W^\top\mathbb E_\pi[e_z\mid x]\right)
+
\mathbb E_\pi\left[
V^y\odot(W^\top e_z)
\mid x
\right].
\end{aligned}
\tag{9}
\]

Furthermore,
\[
\mathbb E_\pi[h_{x,y}\mid x]
=
U^x+\mathbb E_\pi[Ve_y\mid x].
\tag{10}
\]
It follows from (10) that
\[
\begin{aligned}
&\mathbb E_\pi\left[
h_{x,y}\odot(W^\top\mathbf 1_P)
\,\middle|\,x
\right]
\\
&\qquad=
\left(
U^x+\mathbb E_\pi[Ve_y\mid x]
\right)
\odot(W^\top\mathbf 1_P),
\end{aligned}
\tag{11}
\]
and
\[
\begin{aligned}
&\mathbb E_\pi\left[
W^\top\Pi_PWh_{x,y}
\,\middle|\,x
\right]
\\
&\qquad=
W^\top\Pi_PW
\left(
U^x+\mathbb E_\pi[Ve_y\mid x]
\right).
\end{aligned}
\tag{12}
\]

Substituting (8)--(12) into the gradient-flow equation (1), we obtain
\[
\begin{aligned}
\frac{\mathrm dU^x}{\mathrm dt}
=
r_x\Bigg\{
&
c_1W^\top
\left(
\mathbb E_\pi[e_z\mid x]
-\frac1P\mathbf 1_P
\right)
\\
&+
c_2\left(
U^x\odot
\left(W^\top\mathbb E_\pi[e_z\mid x]\right)
+
\mathbb E_\pi\left[
V^y\odot(W^\top e_z)
\mid x
\right]
\right)
\\
&-
\frac{c_2}{P}
\left(
U^x+\mathbb E_\pi[Ve_y\mid x]
\right)
\odot(W^\top\mathbf 1_P)
\\
&-
\frac{c_1^2}{P}
W^\top
\left(
I-\frac1P\mathbf 1_P\mathbf 1_P^\top
\right)
W
\left(
U^x+\mathbb E_\pi[Ve_y\mid x]
\right)
\Bigg\}
+\delta,
\end{aligned}
\]
where $\delta=O (\epsilon^3)$. 

The same calculation would be applied to $V$. Then we obtain the gradient flow for $W$.
Denote $r_z$ by the probability of samples labeled $z$. According to the definition, there is
$$
\frac{\mathrm d W_z^\top}{\mathrm dt}
=(e_z-p_{x,y}) \,\sigma(h_{x,y}),
$$

Then, the gradient flow of the \(z\)-th row of \(W\) is expressed in the expected form as

$$
\begin{aligned}
\frac{\mathrm d W_z^\top}{\mathrm dt}
={}&
c_1\left(
r_z\,\mathbb E_\pi[h_{x,y}\mid Z=z]
-\frac1P\mathbb E_\pi[h_{x,y}]
\right)
\\
&+
\frac{c_2}{2}\left(
r_z\,\mathbb E_\pi[
(h_{x,y}\odot h_{x,y})\mid Z=z]
-\frac1P\mathbb E_\pi[
h_{x,y}\odot h_{x,y}]
\right)
\\
&-
\frac{c_1^2}{P}
\mathbb E_\pi[h_{x,y}^\top h_{x,y}]W^\top\Pi_P e_z
+\delta .
\end{aligned}
$$

Substituting $h_{x,y} = U^x + V^y$ yields a similar result.
\end{proof}

\subsection{Proof of Dominance of the gradient}\label{app:dominance}
\begin{proof}
We give the proof for $U$ and the same calculation would be able to applied to $V$ and $W$.
For convenience, write
\[
q_x:=\mathbb E_\pi[e_z\mid x],
\qquad
\bar V_x:=\mathbb E_\pi[Ve_y\mid x],
\]
and define
\[
\Pi_P
:=
I_P-\frac1P\mathbf 1_P\mathbf 1_P^\top.
\]
From proposition \ref{prop:gradient_flow}, we have
\[
\begin{aligned}
\frac{\mathrm dU^x}{\mathrm dt}
=
r_x\Bigg\{
&
c_1W^\top\left(q_x-\frac1P\mathbf 1_P\right)
\\
&+
c_2U^x\odot(W^\top q_x)
\\
&+
c_2\mathbb E_\pi\left[
V^y\odot(W^\top e_z)\mid x
\right]
\\
&-
\frac{c_2}{P}
\left(U^x+\bar V_x\right)
\odot(W^\top\mathbf 1_P)
\\
&-
\frac{c_1^2}{P}
W^\top\Pi_PW
\left(U^x+\bar V_x\right)
\Bigg\}
+\delta.
\end{aligned}
\tag{1}
\]

Since \(z\mid x\) is uniform on $\Omega$,
\[
q_x
=
\mathbb E_\pi[e_z\mid x]
=
\frac1P\mathbf 1_P.
\tag{2}
\]
Therefore,
\[
c_1W^\top
\left(q_x-\frac1P\mathbf 1_P\right)=0.
\tag{3}
\]
and
\[
c_2U^x \odot W^\top
\left(q_x-\frac1P\mathbf 1_P\right)=0.
\tag{4}
\]
The joint product term is
\[
W^\top\Pi_PW
\left(U^x+\bar V_x\right)
=
O(\epsilon^3).
\]

Substitute all these, we have
\[
\frac{\mathrm dU^x}{\mathrm dt}
=
r_xc_2 \bigg[
\mathbb E_\pi\left[
V^y\odot(W^\top e_z)\mid x
\right]- \overline{V} \cdot \overline{W}
\bigg]
+\delta,
\]

It remains to identify the conditional expectation with the stated matrix expression. By definition,
\[
\begin{aligned}
\operatorname{diag}(V\Phi_x^{(a)}W)
&=
\sum_{y,z}
\mathbb P(Z=z,Y=y\mid X=x)\,
\operatorname{diag}\bigl(V^y(e_z^\top W)\bigr)
\\
&=
\sum_{y,z}
\mathbb P(Z=z,Y=y\mid X=x)\,
\left(V^y\odot W^\top e_z\right)
\\
&=
\mathbb E_\pi\left[
V^y\odot(W^\top e_z)\mid x
\right].
\end{aligned}
\]
adding the weight decay term, we conclude that
\[
\frac{\mathrm dU^x}{\mathrm dt}
=-\lambda U^x+
r_xc_2\bigg[\operatorname{diag}
\left(V\Phi_x^{(a)}W\right) - \overline{V} \cdot \overline{W} \bigg]
+\delta.
\]

\end{proof}

\subsection{Proof of Estimation of noise}\label{app:add-dynamics}

\begin{proof}
Let $G_{\mathcal D}$ be the empirical full-batch vector field.  To avoid
random conditional denominators, for every $(x,y)\in\Omega^2$ let $I_{xy}$ be the
indicator that the pair is included in the training set, and let
$\widetilde Z_{xy}$ be its observed label when $I_{xy}=1$.  Uniform sampling
and uniform label replacement give the exact identity
\[
\mathbb E_{\mathcal D}\!\left[
 \frac{I_{xy}}m\mathbf 1_{\{\widetilde Z_{xy}=z\}}
\right]
=\frac1{P^2}\left((1-\alpha)\mathbf 1_{\{z=x+y\}}
                         +\frac{\alpha}{P}\right).
\tag{1}
\]
The expectation in (1) is over the training pairs, the corrupted subset,
and the replacement labels, with the parameters fixed.

We prove the assertion for $U$ first.  Applying the expansion in
Proposition~\ref{prop:gradient_flow} to the empirical measure yields
\[
\begin{aligned}
G^U_{\mathcal D,x}+\lambda U^x
=\frac{\eta}{m}\sum_{y\in\Omega}I_{xy}\Bigg[{}
&c_1W^\top\left(e_{\widetilde Z_{xy}}-\frac1P\mathbf1_P\right)\\
&+c_2(U^x+V^y)\odot
 W^\top\left(e_{\widetilde Z_{xy}}-\frac1P\mathbf1_P\right)
\Bigg]+\Delta^U_{\mathcal D,x},
\end{aligned}
\tag{2}
\]
For equation (1), summed over $y$, shows
that
\[
\mathbb E_{\mathcal D}\!\left[
\frac1m\sum_y I_{xy}
\left(e_{\widetilde Z_{xy}}-\frac1P\mathbf1_P\right)
\right]=0.
\tag{3}
\]
Consequently, both the first-order term and the part of the second-order
term containing $U^x$ have zero sampling mean.  For the remaining part,
(1) gives
\[
\begin{aligned}
&\mathbb E_{\mathcal D}\!\left[
\frac{\eta c_2}{m}\sum_y I_{xy}V^y\odot
W^\top\left(e_{\widetilde Z_{xy}}-\frac1P\mathbf1_P\right)
\right]\\
&\quad=
\frac{\eta c_2(1-\alpha)}{P}
\left[
P^{-1}\sum_y V^y\odot W^\top e_{x+y}
-\overline{V}\odot\overline{W}
\right].
\end{aligned}
\tag{4}
\]
This has already yielded the main components of the theorem.  Exchanging $U$ and $V$ for another proof.

For the $x$-th output row, set
$\Psi_x^+:=\sum_j e_je_{x-j}^\top$.  The same Taylor expansion gives
\[
G^W_{\mathcal D,x}+\lambda W_x
=\frac{\eta}{m}\sum_{a,b}I_{ab}
\left(\mathbf1_{\{\widetilde Z_{ab}=x\}}-\frac1P\right)
\left[c_1h_{ab}+\frac{c_2}{2}(h_{ab}\odot h_{ab})\right]
+\Delta^W_{\mathcal D,x},
\tag{5}
\]
where $h_{ab}=U^a+V^b$.  The linear term in (5) has zero expectation: on every addition diagonal $a+b=x$, both
$\sum_aU^a$ and $\sum_bV^b$ agree with their uniform-background sums.
The same cancellation removes the $U^a\odot U^a$ and
$V^b\odot V^b$ terms.  The two equal cross terms leave
\[
\frac{\eta c_2(1-\alpha)}{P}
\left[P^{-1}
\sum_aU^a\odot V^{x-a}
-\overline{U}\odot \overline{V}\right].
\tag{6}
\]

It remains to control the finite-sample fluctuation.  Define
$\mathcal E_\bullet$ as the centered fluctuation of the displayed
first- and second-order terms, and include the high order remainder.  For simple random
sampling without replacement, the finite-population variance formula gives
a factor $(1-\rho)/(\rho P^3)$ for a fixed input or output coordinate.
Randomly corrupting an $\alpha$ fraction of the selected labels contributes
at most an additional $\alpha/(\rho P^3)$.  In (2) and (5), each coordinate
of a centered first-order summand is bounded by $2|c_1|\eps$, while each
second-order summand is bounded by $C|c_2|\eps^2$.  The variance formula gives
\[
\operatorname{Var}_{\mathcal D}(\mathcal E_\bullet)
\leq
\frac{C\eta^2}{\rho P^3}
\left[(1-\rho)+\alpha\right]\left(c_1^2\eps^2+c_2^2\eps^4\right).
\]
which is the claimed bound.
\end{proof}

It is noted that the error is caused by the $\rho$ and $\alpha$ during the sampling. If the training set contains too few clean samples, the error between  actual gradient and the predicted term will be very large, which is consistent with our intuition.

\subsection{Proof of dominant frequency and phase locking}\label{app:thm-dominant}
\begin{proof}
The Fourier modes obey
\[
\dot U_r=C\overline{V_r}W_r-\lambda U_r,\quad
\dot V_r=C\overline{U_r}W_r-\lambda V_r,\quad
\dot W_r=CU_rV_r-\lambda W_r.
\]
Set $U_r=e^{-\lambda t}X_r(s)$, $V_r=e^{-\lambda t}Y_r(s)$, and
$W_r=e^{-\lambda t}Z_r(s)$, with
\[
s=\frac{C}{\lambda}(1-e^{-\lambda t}).
\]
Then
\[
\frac{\mathrm dX_r}{\mathrm ds}=\overline{Y_r}Z_r,\quad
\frac{\mathrm dY_r}{\mathrm ds}=\overline{X_r}Z_r,\quad
\frac{\mathrm dZ_r}{\mathrm ds}=X_rY_r.
\]
If $s_r^\star<C/\lambda$ is the undamped blow-up time, its damped time is
\[
t_r^\star=-\frac1\lambda\log\left(1-\frac{\lambda s_r^\star}{C}\right).
\]
This map is increasing, so $s_w^\star<s_i^\star$ implies
$t_w^\star<t_i^\star$.  As $t\uparrow t_w^\star$, the energy of mode $w$
diverges while every competing mode remains finite, and its energy fraction
tends to one.

For phase locking, write $U=Ae^{i\alpha}$, $V=Be^{i\beta}$,
$W=Ge^{i\gamma}$, and $\delta=\alpha+\beta-\gamma$.  Direct separation of
real and imaginary parts gives
\[
\dot\delta=-C\sin\delta\left(\frac{BG}{A}+\frac{AG}{B}+\frac{AB}{G}\right).
\]
While the amplitudes remain nonzero, the factor in parentheses is positive.
Thus $\delta=0$ is stable and $\delta=\pi$ is unstable.  A continuous random
initialization avoids the unstable equilibrium with probability one, yielding
$\delta(t)\to0$ modulo $2\pi$ on the regular continuation interval.
\end{proof}

\subsection{Proof of the early-loss theorem}\label{app:thm-early-loss}
\begin{lemma}
Let \(a,b,c,d\) be i.i.d. random variables and let \(f:\mathbb R\to\mathbb R\) be such that the expectations below exist, then
\[
\mathbb E[f(a+b)f(a+c)]\ge \mathbb E[f(a+b)f(c+d)].
\]
\end{lemma}
\begin{proof} Define
$h(a)=\mathbb E[f(a+b)\mid a],\qquad \mu=\mathbb E[f(a+b)].$

Since \(a+b\) and \(c+d\) are independent,
\[
\mathbb E[f(a+b)f(c+d)]
=\mathbb E[f(a+b)]\,\mathbb E[f(c+d)]
=\mu^2.
\]
Given \(a\), the variables \(b\) and \(c\) are independent, so
\[
\mathbb E[f(a+b)f(a+c)\mid a]
=\mathbb E[f(a+b)\mid a]\,\mathbb E[f(a+c)\mid a]
=h(a)^2.
\]
Hence $\mathbb E[f(a+b)f(a+c)]=\mathbb E[h(a)^2].$

Therefore,
\[
\mathbb E[f(a+b)f(a+c)]-\mathbb E[f(a+b)f(c+d)]
=\mathbb E[h(a)^2]-\mu^2.
\]
But \(\mu=\mathbb E[h(a)]\), so
\[
\mathbb E[h(a)^2]-\mu^2
=\operatorname{Var}(h(a))\ge 0.
\]
Thus $\mathbb E[f(a+b)f(a+c)]\ge \mathbb E[f(a+b)f(c+d)]$.
\end{proof}

For activation functions such as GELU and RELU, the above inequality will hold strictly. It is $\kappa(1)>\kappa(0)$ in the main text.

\textbf{Theorem \ref{thm:early-loss}.}

\begin{proof}
The total number of all pairs with the same label is
\[
K_0=\sum_z n_z^2,
\]
Let $n^{(a)}_{x,y}$ count examples with first coordinate $X=x$ and label $Z=y$,
and define $n^{(b)}_{z,y}$ analogously.  In the complete table,
$n^{(a)}_{x,y}=P\varphi^{(a)}_{x,y}$ and
$n^{(b)}_{z,y}=P\varphi^{(b)}_{z,y}$.  The ordered pairs sharing the first
coordinate and label number
\[
C_a=\sum_{x,y}n^{(a)}_{x,y}(n^{(a)}_{x,y}-1)
=P^2\|\varphi^{(a)}\|_2^2-P^2,
\]
and similarly $C_b=P^2\|\varphi^{(b)}\|_2^2-P^2$.  

We define
\[
\kappa(l)=\mathbb E[\sigma(h_i)\sigma(h_j)]\cdot D, \quad l=0,1,2,
\]
Here $(h_i,h_j)$ are scalar jointly Gaussian, each of variance $2s^2$ and covariance $ls^2$. $s$ denotes the initialization scale. The definition is equivalent to that in the main text.

Thus, We can count it by using the principle of inclusion-exclusion as:
\[
-\mathbb E\dot L/\eta = \frac{1}{N^2} \left(\kappa(0)K_0+(\kappa(1)-\kappa(0))(C_a+C_b)+(\kappa(2)-\kappa(0)) N\right),
\]
The last item is independent of the training distribution.
\end{proof}

\subsection{Propositions of Hadamard basis}\label{app:hadamard}
Let
\[
G=(\mathbb{Z}/2\mathbb{Z})^6
\]
and identify the group operation on $G$ with bitwise XOR. For each
$a\in G$, define the matrix $\Phi_a\in\mathbb{R}^{G\times G}$ by
\[
\Phi_a(y,z)
=
\frac{1}{64}\mathbf{1}_{\{z=a+y\}},
\qquad y,z\in G.
\]
For $u\in G$, define the Walsh character
\[
\chi_u(v)=(-1)^{u\cdot v},
\]
where
\[
u\cdot v
=
\sum_{j=1}^6 u_jv_j
\pmod 2.
\]
Let $H_u\in\mathbb{R}^{G}$ be the normalized Walsh vector
\[
H_u(v)=\frac{1}{8}\chi_u(v).
\]
Then the family $\{\Phi_a:a\in G\}$ is simultaneously diagonalized by the
orthonormal Walsh--Hadamard basis $\{H_u:u\in G\}$. More precisely,
\[
\Phi_a H_u
=
\frac{\chi_u(a)}{64}H_u
=
\frac{(-1)^{u\cdot a}}{64}H_u.
\]
Consequently, if $H$ is the orthogonal matrix whose columns are the vectors
$H_u$, then
\[
H^\top \Phi_a H
=
\operatorname{diag}\left(
\frac{\chi_u(a)}{64}
\right)_{u\in G}
\]
for every $a\in G$.
Indeed, for
$u,u'\in G$,
\begin{align*}
\langle H_u,H_{u'}\rangle
&=
\frac{1}{64}
\sum_{v\in G}
\chi_u(v)\chi_{u'}(v)\\
&=
\frac{1}{64}
\sum_{v\in G}
(-1)^{(u+u')\cdot v}.
\end{align*}
If $u=u'$, then the summand is identically equal to $1$, and hence
\[
\langle H_u,H_u\rangle=1.
\]
If $u\neq u'$, then $u+u'\neq 0$. Choose an element $v_0\in G$ such that
\[
(u+u')\cdot v_0=1.
\]
The map $v\mapsto v+v_0$ is a bijection of $G$, and therefore
\begin{align*}
\sum_{v\in G}(-1)^{(u+u')\cdot v}
&=
\sum_{v\in G}(-1)^{(u+u')\cdot(v+v_0)}\\
&=
- \sum_{v\in G}(-1)^{(u+u')\cdot v}.
\end{align*}
Thus the sum is zero. Hence
\[
\langle H_u,H_{u'}\rangle
=
\mathbf{1}_{\{u=u'\}},
\]
so $\{H_u:u\in G\}$ is an orthonormal basis of $\mathbb{R}^{G}$.
Now fix $a,u\in G$. For every $y\in G$, we have
\begin{align*}
(\Phi_a H_u)(y)
&=
\sum_{z\in G}\Phi_a(y,z)H_u(z)\\
&=
\frac{1}{64}H_u(a+y).
\end{align*}
Using the character identity
\[
\chi_u(a+y)
=
\chi_u(a)\chi_u(y),
\]
we obtain
\begin{align*}
(\Phi_a H_u)(y)
&=
\frac{1}{64}\cdot\frac{1}{8}
\chi_u(a+y)\\
&=
\frac{1}{64}\cdot\frac{1}{8}
\chi_u(a)\chi_u(y)\\
&=
\frac{\chi_u(a)}{64}H_u(y).
\end{align*}
Since this holds for every $y\in G$,
\[
\Phi_a H_u
=
\frac{\chi_u(a)}{64}H_u.
\]
Thus every Walsh vector $H_u$ is a common eigenvector of every matrix
$\Phi_a$, with eigenvalue
\[
\lambda_u(a)=\frac{\chi_u(a)}{64}
=
\frac{(-1)^{u\cdot a}}{64}.
\]
Let $H$ denote the orthogonal matrix whose $u$-th column is $H_u$. Since
the columns of $H$ form a common orthonormal eigenbasis, we have
\[
H^\top\Phi_a H
=
\operatorname{diag}\left(
\frac{\chi_u(a)}{64}
\right)_{u\in G}.
\]
Therefore, the matrices $\{\Phi_a:a\in G\}$ are simultaneously
orthogonally diagonalized by the normalized Walsh--Hadamard matrix.

\section{Experiment of Various Label Noise}\label{app:early_noise_prediction}

\subsection{Experiment on Modular Multiplication Task}\label{app:multiplication}

\begin{figure}[t]
  \centering
  \includegraphics[width=\linewidth]{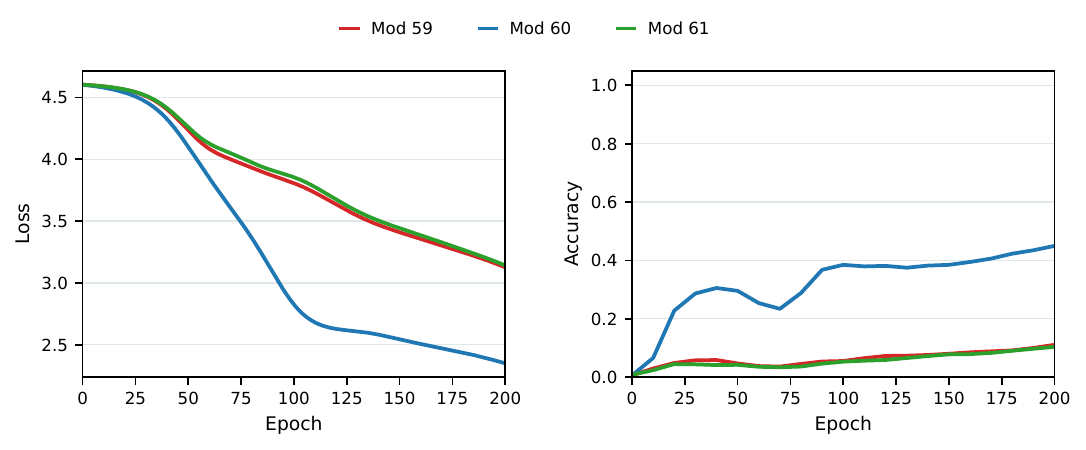}
  \caption{Early-training loss curves of the modular multiplication by $59,60,61$.}
  \label{fig:mul}
\end{figure}

In order to verify our theory, we designed the following experiment on the modular multiplication task.
Consider $N=61,\Omega=\{1,2,\ldots,N\}$ and there are three modular-period that $M_1=59,M_2=60,M_3=61$.
The total pairs of $\Omega\times \Omega$ are divided into three parts $S_1,S_2,S_3$. The labels of data from
\[
S_i=\{(x_i,y_i)\}, \quad i=1,2,3,
\]
are set to
\[
\tilde{z_i}=x_i\cdot y_i \mod M_i, \quad \tilde{z_i}\in \Omega \cup \{0\},
\]
We trained it using the same neural network as before. See Figure~\ref{fig:mul} for the early-training loss curves of the three tasks. We can see that the task with $M_2=60$ showed a significantly faster decline than the others.

Our theory can easily explain this point. For $M=60$, there exist a large number of pairs $(x,y)$ such that $x\cdot y \equiv x' \cdot y \pmod M$, such as
\[
5 \times 15 \equiv 13 \times 15 \pmod {60},
\]
which means that the label $z$ is more concentrated. This leads to larger $K_0=\sum_z n_z^2$ and larger $\|\varphi^{(a)}\|_2^2+\|\varphi^{(b)}\|_2^2$ , thus a faster early training loss decline according to Theorem~\ref{thm:early-loss}. In contrast, for prime numbers $M_1=59$ and $M_3=61$, it satisfies $x\cdot y \equiv x' \cdot y \pmod M$ if and only if $x\equiv x' \pmod M$ for non-zero $y$, which means the signature of $M=59,61$ are more uniformly distributed. This leads to a smaller $K_0$ and $K_1$ (we think that the values of $M_1 = 59$ and $M_3 = 61$ are very close to each other). As shown in Figure ~\ref{fig:mul}, the theoretical prediction is consistent with the actual observation. The early loss slope of $M_2=60$ is much greater than the others.

\subsection{Experiment on the Bijective Noise}\label{app:bijective}
In the main text, we point out that in the scenario of uniform noise on the labels, during the early training phase, it is always observed that the loss of the noise decreases first, as in figure~\ref{fig:early_clean_noise}. 

\begin{figure}[t]
  \centering
  \includegraphics[width=\linewidth]{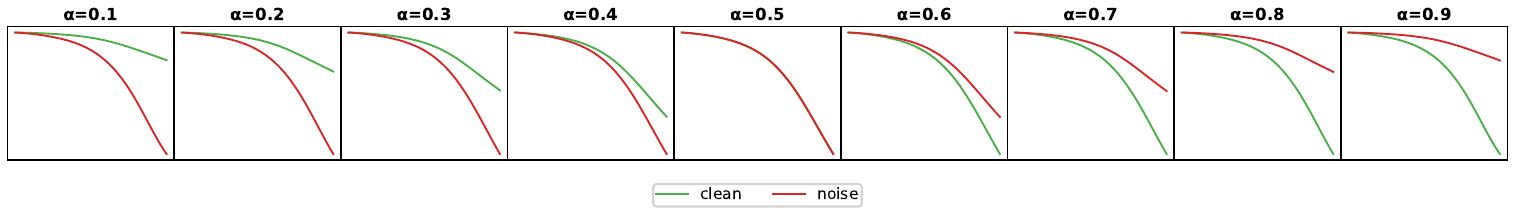}
  \caption{Early-training loss curves of the bijective noise experiment.}
  \label{fig:latin}
\end{figure}

Here we present a specific noise construction that ensures it does not undergo such collisions. Just as our theory predicted, the phenomenon observed after replacing this kind of noise was completely different from what was observed before, as shown in Figure ~\ref{fig:latin}. Specifically, we construct a binary function $g(x,y)$ such that:

\begin{itemize}
    \item $g(x,\cdot)$ is bijective for any fixed $x$ and $g(\cdot,y)$ is bijective for any fixed $y$. 
    \item $g$ is not isomorphic to addition, that is,
    \[
    g(x,y)\neq \sigma_3(\sigma_1(x)+\sigma_2(y)),
    \]
    where $\sigma_i$ is from the permutation group of $\Omega$.
    \item Generally, $g$ cannot be accurately fitted by our 2L-MLP.

\end{itemize}

Our final construction is
\[
g(x,y)=\begin{cases}
  x,& x=y, \\
  (x+(4+\left( \frac{y-x}{113} \right))\cdot (y-x)) \mod 113, & x\neq y,
\end{cases}
\]
where $\left( \frac{\cdot}{113} \right)$ is the Quadratic Residue Legendre symbol (of modular 113). It satisfies the above three properties. Next, we replace the original uniform random noise with the function $g(x,y)$. In this case, the loss of $g$ does not always decrease first during the early training phase. By calculating the coefficient preceding $\kappa(1)$, We can give
\[
\Delta_{\text{clean}}-\Delta_{\text{g}}=2\alpha-1,
\]
Inferring like that in chapter \ref{chapter:puzzle}, a larger $\Delta$ gives a more negative early loss derivative. This indicates that when the noise ratio $\alpha < 0.5$, the noise data ($g(x,y)$) decrease first. In contrast, when $\alpha > 0.5$, the opposite occurs. As shown in Figure ~\ref{fig:latin}, the theoretical prediction is consistent with the actual observation.

\end{document}